\documentclass[11pt]{article}

\usepackage[margin=1in]{geometry}
\usepackage{amsmath,amssymb,amsthm,mathtools}
\usepackage{enumitem}
\usepackage{booktabs}
\usepackage{array}
\usepackage{tabularx}
\usepackage{aliascnt}
\usepackage{microtype}
\usepackage[colorlinks=true,linkcolor=blue,citecolor=blue,urlcolor=blue]{hyperref}
\usepackage[nameinlink,noabbrev]{cleveref}
\title{Active-Trace Complexity Bounds for Moreau--Yosida Unadjusted Langevin Sampling}
\author{Yuchen Xin\textsuperscript{*},
Zhihua Zhang\textsuperscript{$\dagger$}}
\date{}

\newtheorem{theorem}{Theorem}[section]

\newaliascnt{proposition}{theorem}
\newtheorem{proposition}[proposition]{Proposition}
\aliascntresetthe{proposition}

\newaliascnt{lemma}{theorem}
\newtheorem{lemma}[lemma]{Lemma}
\aliascntresetthe{lemma}

\newaliascnt{corollary}{theorem}
\newtheorem{corollary}[corollary]{Corollary}
\aliascntresetthe{corollary}

\newaliascnt{assumption}{theorem}
\newtheorem{assumption}[assumption]{Assumption}
\aliascntresetthe{assumption}

\theoremstyle{definition}
\newaliascnt{definition}{theorem}
\newtheorem{definition}[definition]{Definition}
\aliascntresetthe{definition}

\theoremstyle{remark}
\newaliascnt{remark}{theorem}
\newtheorem{remark}[remark]{Remark}
\aliascntresetthe{remark}

\numberwithin{equation}{section}

\crefname{theorem}{Theorem}{Theorems}
\Crefname{theorem}{Theorem}{Theorems}
\crefname{lemma}{Lemma}{Lemmas}
\Crefname{lemma}{Lemma}{Lemmas}
\crefname{assumption}{Assumption}{Assumptions}
\Crefname{assumption}{Assumption}{Assumptions}
\crefname{definition}{Definition}{Definitions}
\Crefname{definition}{Definition}{Definitions}
\crefname{remark}{Remark}{Remarks}
\Crefname{remark}{Remark}{Remarks}
\crefname{proposition}{Proposition}{Propositions}
\Crefname{proposition}{Proposition}{Propositions}
\crefname{corollary}{Corollary}{Corollaries}
\Crefname{corollary}{Corollary}{Corollaries}
\crefname{appendix}{Appendix}{Appendices}
\Crefname{appendix}{Appendix}{Appendices}

\DeclareMathOperator{\prox}{prox}
\DeclareMathOperator{\Ent}{Ent}
\DeclareMathOperator{\KL}{KL}
\DeclareMathOperator{\tr}{tr}
\DeclareMathOperator{\dist}{dist}
\DeclareMathOperator{\Lip}{Lip}
\DeclareMathOperator{\Var}{Var}
\DeclareMathOperator*{\argmin}{arg\,min}

\newcommand{\R}{\mathbb{R}}
\newcommand{\cP}{\mathcal{P}}

\newcommand{\Id}{I}
\newcommand{\dd}{\,\mathrm{d}}
\newcommand{\law}{\mathcal{L}}
\newcommand{\E}{\mathbb{E}}
\newcommand{\eps}{\varepsilon}
\newcommand{\norm}[1]{\left\lVert #1\right\rVert}
\newcommand{\ip}[2]{\left\langle #1,#2\right\rangle}
\newcommand{\W}{W_2}
\newcommand{\TV}{\mathrm{TV}}

\newcommand{\Pp}{\mathbb{P}}
\newcommand{\vol}{\operatorname{vol}}
\newcommand{\op}{\mathrm{op}}
\newcommand{\rank}{\mathrm{rank}}

\providecommand{\MYULA}{\textsc{myula}}
\providecommand{\PGLA}{\textsc{pgla}}
\providecommand{\PSGLA}{\textsc{psgla}}

\begin{document}
\maketitle

\begingroup
\renewcommand{\thefootnote}{\fnsymbol{footnote}}
\footnotetext[1]{School of Mathematical Sciences, Peking University; email: \texttt{2301110087@pku.edu.cn}}
\footnotetext[2]{School of Mathematical Sciences, Peking University; email: \texttt{zhzhang@math.pku.edu.cn}}
\endgroup

\begin{abstract}
We study the Moreau--Yosida unadjusted Langevin algorithm (MYULA) for the nonsmooth composite target \[ \pi(dx)\propto \exp\{-f(x)-g(x)\}\,dx, \qquad x\in\mathbb R^d, \] where \(f\) is \(m\)-strongly convex with \(L_f\)-Lipschitz gradient and \(g\) is convex and \(G\)-Lipschitz. Let \(g_\lambda\) be the Moreau envelope of \(g\), \(\pi_\lambda\) the corresponding smoothed target, and \(a_\lambda=\operatorname{tr}H_\lambda\), where \(H_\lambda\) is the a.e./weak Hessian of \(g_\lambda\). We show that the leading MYULA discretization error is controlled by the reference active trace \(B_{\mathrm{ref}}\), the average of \(a_\lambda\) along the heat substep of one MYULA update started from \(\pi_\lambda\), rather than by the global curvature bound \(d/\lambda\). If \(M_\lambda\) is an a.e. upper bound for \(a_\lambda\), then, up to logarithmic factors, \[ N \lesssim \frac{1}{m} \left[ L_f + \frac{ \tau_f+G^2+B_{\mathrm{ref}} }{ \varepsilon_{\mathrm{alg}}^2 } + \frac{M_\lambda}{\varepsilon_{\mathrm{alg}}} \right], \qquad \tau_f:= \sup_x\operatorname{tr}\nabla^2 f(x), \] iterations suffice to ensure \(\sqrt m\,W_2(\mu_N,\pi_\lambda)\leq\varepsilon_{\mathrm{alg}}\), where \(\mu_N\) is the law of the \(N\)-th iterate and \(W_2\) is the quadratic Wasserstein distance. We also prove the Moreau-bias bound \[ \sqrt m\,W_2(\pi_\lambda,\pi) \leq \frac{G^2\lambda}{4}. \] Thus, choosing \(\lambda\asymp\varepsilon/G^2\) gives an end-to-end guarantee for \(\pi\). The universal estimate \(B_{\mathrm{ref}}\leq d/\lambda\) yields \(\widetilde O(\varepsilon^{-3})\) accuracy dependence. For the structured piecewise-linear, lasso-type, group, and total-variation penalties considered here, curvature--tube estimates make \(B_{\mathrm{ref}}\) independent of \(\lambda\), yielding \(\widetilde O(\varepsilon^{-2})\) for the same classical MYULA kernel.
\end{abstract}

\tableofcontents

\section{Introduction}
\label{sec:introduction}

Sampling from a probability distribution with density
\begin{equation*}
    \pi(x) \propto \exp\{-f(x)-g(x)\}, \qquad x\in\mathbb R^d,
\end{equation*}
is a fundamental computational task in Bayesian inverse problems, high-dimensional
statistics, and machine learning.  The composite form is particularly common
when the smooth term $f$ represents a data-fidelity or negative
log-likelihood and the convex, possibly nonsmooth term $g$ encodes structural
information such as sparsity, group sparsity, analysis sparsity, or total
variation.  Proximal optimization methods can exploit this structure directly,
but conventional gradient-based Langevin algorithms require a differentiable
potential with a sufficiently regular gradient.  This mismatch has motivated
a broad class of proximal and smoothing-based Markov chain Monte Carlo
methods; see, among others, \cite{pereyra2016proximal,durmus2018efficient,durmus2019analysis}.

A particularly influential method is the Moreau--Yosida unadjusted Langevin
algorithm (\MYULA) of \cite{durmus2018efficient}.  It replaces $g$
by its Moreau envelope $g_\lambda$, targets the smoothed measure
$\pi_\lambda$ (see \eqref{eq:smoothed-target}), and applies the Euler--Maruyama
scheme (see \eqref{eq:myula-update}).  The resulting drift is explicit whenever one can
evaluate $\nabla f$ and $\prox_{\lambda g}$, since
\[
    \nabla g_\lambda(x)
    = \lambda^{-1}\bigl(x-\prox_{\lambda g}(x)\bigr).
\]
This construction has proved especially useful in imaging and sparse Bayesian
inference because it converts proximal primitives already available for
optimization into a simple Langevin sampler.

The standard smoothness description of the Moreau envelope is nevertheless
potentially pessimistic.  Globally,
\[
    \operatorname{Lip}(\nabla g_\lambda)\leq \lambda^{-1},
    \qquad
    0\preceq \nabla^2 g_\lambda(x)\preceq \lambda^{-1}\Id
    \quad\text{for a.e. }x,
\]
so a worst-case analysis treats the full trace of the Moreau curvature as
being as large as $d/\lambda$ everywhere.  At the same time, decreasing
$\lambda$ is necessary to reduce the discrepancy between $\pi_\lambda$ and
the original target $\pi$.  This creates the familiar tension between
regularization bias and discretization stability: a smaller smoothing
parameter improves the target approximation but apparently makes the
Langevin dynamics uniformly stiffer.

For many structured penalties, however, the curvature $\lambda^{-1}$ is not
present throughout the state space.  For the scalar absolute value, for
example, $g_\lambda''(x)=\lambda^{-1}$ only inside a threshold interval of
width $O(\lambda)$ and vanishes outside that interval.  Coordinatewise
piecewise-linear penalties generate thin slabs; group penalties generate a
small central ball together with integrable tangential curvature; and
polyhedral analysis penalties generate neighborhoods of active faces.  Thus,
the largest Moreau curvature and the probability of encountering it are
coupled.  A bound based only on $\sup_x\|\nabla^2g_\lambda(x)\|$ discards this
coupling and charges every Langevin step for curvature that may be visited
only with probability $O(\lambda)$.

This paper develops a distribution-dependent analysis that retains that
coupling.  Let $H_\lambda$ be a measurable representative of the a.e./weak
Hessian of the $C^{1,1}$ function $g_\lambda$.  We
introduce the \emph{Moreau active trace} (see Definition~\ref{def:active-trace})
\begin{equation*}
    a_\lambda(x)
    := \operatorname{tr}H_\lambda(x)
    = \frac{1}{\lambda}
      \operatorname{tr}\!\left(\Id-D\prox_{\lambda g}(x)\right)
    \quad\text{for a.e. }x.
\end{equation*}
In addition to being the trace of the local Moreau curvature, the quantity
$\lambda a_\lambda(x)$ is the total local shrinkage of the proximal map. Therefore, it describes how many directions are locally suppressed by the active
structure of $g$.  Although the universal bound
$a_\lambda\leq d/\lambda$ always holds, the distributional average of
$a_\lambda$ can be much smaller.

The relevant average in our analysis is determined by one \MYULA{} step.  If
$X_{\mathrm{ref}}\sim\pi_\lambda$, the deterministic Euler map is
$T_h(x)=x-h\nabla(f+g_\lambda)(x)$, and $W_t$ is a Brownian motion independent of
$X_{\mathrm{ref}}$. 
Define
\begin{equation*}
    B_{\mathrm{ref}}
    := \frac{1}{h}\int_0^h
       \mathbb E\!
       \left[a_\lambda\!\left(T_h(X_{\mathrm{ref}})+\sqrt{2}W_t\right)\right]
       \,\mathrm dt,
\end{equation*}
which we call the \emph{reference heat-path active trace} 
(see Definition~\ref{def:reference-active-trace} ). The definition is
algorithm-aware: it averages the weak Moreau curvature over the stationary
input, the Euler drift, and the Gaussian heat interpolation that together
form a single discretization step.  It is also the quantity that appears
naturally when the change of the potential along the heat step is written as
an integral of its weak Laplacian.

Our main theorem shows that this pathwise average, rather than the global
trace bound, controls the leading nonsmooth discretization term. In particular, under the
standing assumptions of Section~\ref{sec:setup},
Theorem~\ref{thm:main-complete} reads
\begin{equation*}
  \Phi_N
  \lesssim
  e^{-mhN/2}\Phi_0
  +\frac{1}{m}
   \left\{(\tau_f+G^2+B_{\mathrm{ref}})h+M_\lambda^2h^2\right\},
  \qquad
  \Phi_k=W_2^2(\mu_k,\pi_\lambda)
         +2h\KL(\mu_k\Vert\pi_\lambda),
\end{equation*}
whenever $0\leq a_\lambda\leq M_\lambda$ a.e.  Equivalently, the
fixed-$\lambda$ iteration complexity 
(see \eqref{eq:complexity-summary}) is, up to logarithmic factors,
\begin{equation*}
    N_\lambda(\varepsilon_{\mathrm{alg}})
    \lesssim
    \frac{1}{m}
    \left(
       L_f
       +\frac{\tau_f+G^2+B_{\mathrm{ref}}}{\varepsilon_{\mathrm{alg}}^2}
       +\frac{M_\lambda}{\varepsilon_{\mathrm{alg}}}
    \right).
\end{equation*}
The leading $\varepsilon_{\mathrm{alg}}^{-2}$ term depends on the reference
active trace; the worst-case curvature $M_\lambda$ survives only in a lower
order transfer term.  Proposition~\ref{prop:general-bias} separately proves
\[
    \sqrt m\,W_2(\pi_\lambda,\pi)
    \leq \frac{G^2\lambda}{4},
\]
so the fixed-$\lambda$ estimate can be combined with a transparent
regularization-bias choice.

This decomposition makes the structural gain explicit.  With the symmetric
error split and $\lambda\asymp\varepsilon/G^2$, replacing $B_{\mathrm{ref}}$ by the
universal bound $d/\lambda$ in our own theorem produces a conservative
contribution of order $dG^2\varepsilon^{-3}$.  In contrast, if $B_{\mathrm{ref}}$ is
bounded independently of $\lambda$, this cubic contribution disappears and
the structured examples of Section~\ref{sec:examples-verification} have
$\widetilde{O}(\varepsilon^{-2})$ dependence on the target accuracy, with the
remaining dimension and geometry factors displayed explicitly in
Propositions~\ref{prop:oned-pl-reference-trace}--\ref{prop:generalized-lasso-reference-trace} and
Corollaries~\ref{cor:weighted-lasso-reference-trace} and~\ref{cor:anisotropic-tv-reference-trace}.

Section~\ref{sec:related-myula} places this gain against a
metric- and assumption-aligned global-smoothness benchmark.
Applying the Wasserstein--EVI analysis of
\cite[Corollary~10]{durmus2019analysis} to the same smoothed potential
\(U_\lambda\), followed by the end-to-end choice
\(\lambda\asymp\varepsilon/G^2\), gives
\(\widetilde O(\varepsilon^{-3})\) dependence, matching the universal
global-trace specialization of our theorem.
For the structured penalties considered here, the active-trace bounds
replace this global curvature charge by an occupation-weighted quantity
and yield \(\widetilde O(\varepsilon^{-2})\) without changing the
classical \MYULA{} transition.
The same section explains why the
\(\varepsilon^{-2}\)-type total-variation bounds of
\cite{durmus2018efficient} are fixed-\(\lambda\) statements rather than
\(\lambda\)-free end-to-end rates.

The mechanism for bounding $B_{\mathrm{ref}}$ is geometric.  Proposition~\ref{prop:reftrace-tube-control}
formalizes a curvature--tube interface: if a component of $a_\lambda$ has
height $O(\lambda^{-1})$ inside an $O(\lambda)$ neighborhood of an active
stratum $\Sigma$, and the reference heat path assigns mass $O(r^q)$ to an
$r$-tube around $\Sigma$, then this component contributes
$O(\lambda^{q-1})$ to $B_{\mathrm{ref}}$.  Codimension-one layers therefore contribute
$O(1)$, while higher-codimension central regions can contribute even less as
$\lambda\downarrow0$.  The slice-density and density-propagation estimates in
Lemmas~\ref{lem:reftrace-slice-density} and~\ref{lem:reftrace-density-propagation} convert this
interface into verifiable bounds without invoking the global Moreau
smoothness scale.

The examples illustrate several forms of active geometry.  For
one-dimensional finite-kink and separable piecewise-linear penalties, the
Moreau curvature is confined to intervals or coordinate slabs whose total
width is $O(\lambda)$.  For the group lasso, a small-ball estimate controls
the central $q/\lambda$ curvature and an inverse-radius estimate controls the
tangential part.  For generalized lasso and anisotropic total variation, the
proximal map is affine on polyhedral cells, with
\[
    D\prox_{\lambda g}(x)=P_{\ker D_A},
    \qquad
    a_\lambda(x)=\frac{\operatorname{rank}(D_A)}{\lambda},
\]
where $A$ is the active row set; the corresponding curvature is localized by
thin row slabs.  These calculations connect the sampling error to active-set
rank, block geometry, and tube probabilities rather than to a uniform
$d/\lambda$ penalty.

Conceptually, the closest smooth analogue is the recent
average-smoothness analysis of Langevin Monte Carlo
\cite{dalalyan2026improved}.  That work shows that the leading LMC
discretization term can depend on an average of coordinatewise smoothness
constants instead of the largest global smoothness constant.  Both results
reflect the fact that isotropic Gaussian noise naturally interacts with a
trace-type measure of curvature.  The distinction is that average smoothness
averages over directions while retaining a supremum over spatial locations.
The reference active trace additionally averages over the locations actually
visited by the one-step heat path.  This extra spatial averaging is essential
for Moreau envelopes of lasso-type penalties: every active coordinate may
have worst-case curvature $1/\lambda$, even though that curvature is confined
to a slab with probability $O(\lambda)$.

Our objective is neither to propose a new sampler nor to give the first
$\widetilde{O}(\varepsilon^{-2})$ result for nonsmooth composite sampling.
Proximal splitting Langevin methods such as \PGLA{} and \PSGLA{} update the
nonsmooth component directly and admit strong guarantees for the original
composite target \cite{durmus2019analysis,salim2020primal}.
Other methods use alternative envelopes, subgradient dynamics,
Metropolis corrections, or stronger restricted-Gaussian oracles.  These
methods change the transition kernel, the target treated by each step, or the
oracle model.  In contrast, the present work keeps the classical \MYULA{}
iteration and asks a narrower question: \emph{which part of the Moreau
curvature is actually paid for by its discretization error?}

The main contributions can be summarized as follows.
\begin{enumerate}
    \item We derive a universal KL--EVI recursion for \MYULA{} in which the
    deterministic gradient-step error depends on the Lipschitz magnitude
    $G^2$ of the nonsmooth penalty, rather than on a power of
    $\lambda^{-1}$; see Lemma~\ref{lem:gradient-evi} and
    Theorem~\ref{thm:active-trace-recursion}.

    \item We introduce the reference heat-path active trace and close the
    recursion through an active-trace transfer argument; see
    Proposition~\ref{prop:reference-TV-transfer} and
    Theorem~\ref{thm:abstract-complexity}.  Together with the Wasserstein
    Moreau-bias estimate in Proposition~\ref{prop:general-bias}, this yields a
    complete guarantee for the original nonsmooth target.

    \item We develop a curvature--tube principle and accompanying
    slice-density estimates that turn localization of weak Moreau curvature
    into quantitative bounds on $B_{\mathrm{ref}}$; see
    Section~\ref{sec:reftrace}.

    \item We verify the framework for finite-kink piecewise-linear penalties,
    weighted lasso, group lasso, generalized lasso, and anisotropic total
    variation.  In these examples $B_{\mathrm{ref}}$ is bounded independently of
    $\lambda$, removing the cubic accuracy contribution generated by the
    universal global-trace substitution.
\end{enumerate}

The remainder of the paper is organized as follows. Section~\ref{sec:related-work} presents the related work.  
Section~\ref{sec:preliminaries} fixes the
nonsmooth second-order conventions and the heat-flow tools used throughout.
Section~\ref{sec:setup} introduces the composite target, the
\MYULA{} update, and the
standing assumptions.  Section~\ref{sec:main-result} proves the active-trace
recursion, closes it
through the reference-path transfer, and controls the Moreau approximation
bias.  Section~\ref{sec:reftrace} develops the
curvature--tube and density-propagation
machinery.  Section~\ref{sec:examples-verification} treats the structured examples, and
the appendices provide the real-analysis details needed for a.e. and weak Hessians.

\section{Related Work}
\label{sec:related-work}

\subsection{Proximal MCMC and the original analysis of \MYULA}
\label{sec:related-myula}

\paragraph{Proximal MCMC and the original MYULA bounds.}
Proximal MCMC was introduced by \cite{pereyra2016proximal} to make
gradient-based sampling applicable to log-concave but nonsmooth models whose
proximal mappings are computationally accessible.  The Moreau--Yosida
unadjusted Langevin algorithm (\MYULA{}) of
\cite{durmus2018efficient} applies ULA to
\[
    U_\lambda=f+g_\lambda,
    \qquad
    \pi_\lambda(\mathrm{d}x)
    \propto
    \exp\{-U_\lambda(x)\}\,\mathrm{d}x.
\]
Its classical analysis uses the global regularity estimate
\begin{equation}
    L_\lambda
    :=
    \operatorname{Lip}(\nabla U_\lambda)
    \leq
    L_f+\lambda^{-1},
    \qquad
    \gamma
    \leq
    L_\lambda^{-1}
    =
    \frac{\lambda}{1+\lambda L_f},
    \label{eq:rw-original-myula-smoothness}
\end{equation}
and, when \(g\) is \(G\)-Lipschitz, the regularization bound
\begin{equation}
    \|\pi_\lambda-\pi\|_{\mathrm{TV}}
    \leq
    \lambda G^2.
    \label{eq:rw-original-myula-tv-bias}
\end{equation}

Let $\mu_n$ be the law of the MYULA iterate, and let
\(\varepsilon_{\mathrm{alg}}\) denote the prescribed error to the fixed smoothed target:
\begin{equation}
    \|\mu_n-\pi_\lambda\|_{\mathrm{TV}}
    \leq
    \varepsilon_{\mathrm{alg}}.
    \label{eq:rw-original-myula-fixed-lambda-error}
\end{equation}

There is a minor notational carryover in the presentation of
Theorems~2--3 of \cite{durmus2018efficient}. Their displayed
total-variation conclusions use the symbol \(\pi\), but the transition
kernel \(R_\gamma\) is generated by \(U_\lambda=f+g_\lambda\), and the
proofs apply the generic smooth-ULA results of
\cite{durmus2017nonasymptotic} with \(U=U_\lambda\). The target after
this specialization is therefore the corresponding Gibbs law
\(\pi_\lambda\), with the approximation error
\(||\pi_\lambda-\pi||_{\mathrm{TV}}\) controlled separately in
Proposition~1. Accordingly, we read the displayed conclusions as
fixed-\(\lambda\) bounds to \(\pi_\lambda\), not as end-to-end bounds to
\(\pi\).

Under this proof-supported interpretation, and renaming the prescribed
fixed-target tolerance as \(\varepsilon_{\mathrm{alg}}\), \cite{durmus2018efficient} reports the worst-case iteration dependences
\[
    O\!\left(
        d^5
        \log^2(\varepsilon_{\mathrm{alg}}^{-1})
        \varepsilon_{\mathrm{alg}}^{-2}
    \right)
\]
under its general coercive-tail condition H3, and
\[
    O\!\left(
        d\log(d)
        \log^2(\varepsilon_{\mathrm{alg}}^{-1})
        \varepsilon_{\mathrm{alg}}^{-2}
    \right)
\]
under the stronger tail-convexity condition H4.  These displayed
orders describe the dependence on the fixed-\(\lambda\) algorithmic
tolerance; the smoothness parameter \(L_\lambda\), together with the
initialization, Lyapunov, tail, and convexity constants in the detailed
bounds, remains part of the problem dependence.

\paragraph{Fixed-\(\lambda\) versus end-to-end accuracy.}
A guarantee for the original nonsmooth target requires combining the
fixed-\(\lambda\) algorithmic error with the Moreau regularization bias:
\begin{equation}
\begin{aligned}
    \|\mu_n-\pi\|_{\mathrm{TV}}
    &\leq
    \|\mu_n-\pi_\lambda\|_{\mathrm{TV}}
    +
    \|\pi_\lambda-\pi\|_{\mathrm{TV}}
    \\
    &\leq
    \varepsilon_{\mathrm{alg}}
    +
    \lambda G^2.
\end{aligned}
    \label{eq:rw-original-myula-total-error}
\end{equation}
For example, to obtain a total tolerance \(\varepsilon\), a
constant-fraction error allocation is
\begin{equation}
    \varepsilon_{\mathrm{alg}}
    =
    \frac{\varepsilon}{2},
    \qquad
    \lambda
    \leq
    \frac{\varepsilon}{2G^2},
    \qquad
    L_{\lambda(\varepsilon)}
    =
    O\!\left(
        L_f+\frac{G^2}{\varepsilon}
    \right).
    \label{eq:rw-original-myula-end-to-end-choice}
\end{equation}
Consequently, the published
\(\varepsilon_{\mathrm{alg}}^{-2}\) summaries in \cite{durmus2018efficient} should not be read as
\(\lambda\)-free end-to-end
\(\widetilde O(\varepsilon^{-2})\) rates for the original target.
For example, the explicit $L_{\lambda(\varepsilon)}^2\varepsilon_{\mathrm{alg}}^{-2}$ factor in the underlying step-size bound formally becomes
\begin{equation}
\begin{aligned}
    O\!\left[
        \left(
            L_f+\frac{G^2}{\varepsilon}
        \right)^2
        \varepsilon^{-2}
    \right],
\end{aligned}
    \label{eq:rw-original-myula-schematic-composition}
\end{equation}
whose highest-order visible term is $O(G^4\varepsilon^{-4})$. This is only a schematic substitution: the remaining constants in \cite{durmus2018efficient} were not optimized uniformly in $\lambda$.

For compactly supported log-concave targets, \cite{brosse2017sampling} performs the smoothing–discretization balance explicitly. There $g$ is an indicator function, the smoothing bias has a different form, and $\lambda$ is chosen of order $\epsilon^2$, up to dimension and geometric factors, leading to \(\widetilde O(\varepsilon^{-6})\) complexity in both total variation and \(W_1\).

\paragraph{A global-smoothness benchmark in the Wasserstein--EVI framework.}
A closely aligned benchmark follows directly from
\cite[Corollary~10]{durmus2019analysis}: after setting
\(U=U_\lambda\), it applies to the same ULA kernel and smoothed target
in \(W_2\) under strong convexity, and its underlying proof uses the
same gradient-step/heat-step decomposition. For an \(m\)-strongly convex
potential with \(L\)-Lipschitz gradient, that result states that, for any
squared-Wasserstein tolerance \(\delta>0\), the constant-step choices
\begin{equation}
    h_\delta
    \leq
    \min\left\{
        \frac{m\delta}{4Ld},
        \frac{1}{L}
    \right\},
    \qquad
    N_\delta
    \geq
    \frac{1}{mh_\delta}
    \log\left(
        \frac{
            2W_2^2(\mu_0,\pi)
        }{
            \delta
        }
    \right)
    \label{eq:rw-dmm-direct-complexity}
\end{equation}
ensure
\[
    W_2^2(\mu_{N_\delta},\pi)
    \leq
    \delta.
\]
Here \(\delta\) is a tolerance for \(W_2^2\), whereas our accuracy
parameter controls \(\sqrt m\,W_2\).

Applying \eqref{eq:rw-dmm-direct-complexity} to
\(U_\lambda=f+g_\lambda\), with
\[
    L=L_\lambda
    \leq
    L_f+\lambda^{-1},
\]
and setting
\[
    \delta
    =
    \frac{\varepsilon_{\mathrm{alg}}^2}{m}
\]
gives
\begin{equation}
    h
    \lesssim
    \min\left\{
        L_\lambda^{-1},
        \frac{
            \varepsilon_{\mathrm{alg}}^2
        }{
            dL_\lambda
        }
    \right\}
    \label{eq:rw-dmm-myula-step}
\end{equation}
and, up to logarithmic factors,
\begin{equation}
\begin{aligned}
    N_\lambda^{\mathrm{glob}}
    (\varepsilon_{\mathrm{alg}})
    =
    \widetilde O\left[
        \frac{L_\lambda}{m}
        \left(
            1+
            \frac{d}{\varepsilon_{\mathrm{alg}}^2}
        \right)
    \right]
    =
    \widetilde O\left[
        \frac{
            L_f+\lambda^{-1}
        }{m}
        \left(
            1+
            \frac{d}{\varepsilon_{\mathrm{alg}}^2}
        \right)
    \right].
\end{aligned}
    \label{eq:rw-dmm-fixed-lambda-myula}
\end{equation}

To obtain an end-to-end guarantee for the original nonsmooth target, combine
\eqref{eq:rw-dmm-fixed-lambda-myula} with the Wasserstein Moreau-bias
estimate in Theorem~\ref{thm:main-complete},
\begin{equation}
    \sqrt m\,W_2(\pi_\lambda,\pi)
    \leq
    \frac{G^2\lambda}{4}.
    \label{eq:rw-our-bias-for-global-benchmark}
\end{equation}
With the symmetric allocation
\[
    \varepsilon_{\mathrm{alg}}
    =
    \varepsilon_{\mathrm{bias}}
    =
    \frac{\varepsilon}{2},
    \qquad
    \lambda
    \asymp
    \frac{\varepsilon}{G^2},
\]
equation~\eqref{eq:rw-dmm-fixed-lambda-myula} yields
\begin{equation}
\begin{aligned}
    N_{\mathrm{glob}}(\varepsilon)
    =
    \widetilde O\left[
        \frac{1}{m}
        \left(
            L_f+\frac{G^2}{\varepsilon}
        \right)
        \left(
            1+\frac{d}{\varepsilon^2}
        \right)
    \right]
    =
    \widetilde O\left(
        \frac{dL_f}{m\varepsilon^2}
        +
        \frac{dG^2}{m\varepsilon^3}
    \right),
\end{aligned}
    \label{eq:rw-direct-global-epsilon-cubic}
\end{equation}
where lower-order terms are suppressed for
\(d\geq1\) and \(0<\varepsilon\leq1\).

Thus, within the Wasserstein--EVI framework of
\cite{durmus2019analysis}, applying the global smoothness
\(L_\lambda\leq L_f+\lambda^{-1}\) to the Moreau-smoothed target yields
the end-to-end benchmark
\(\widetilde O(\varepsilon^{-3})\).

\paragraph{Active-trace refinement within the same proof architecture.}
In \cite{durmus2019analysis}, the global constant \(L\) enters both the
deterministic-step stability estimate and the heat-step energy
increment; for \(U=U_\lambda\), both are therefore charged through
\(L_\lambda\leq L_f+\lambda^{-1}\).
Our analysis separates these two roles: the deterministic step uses
\(\|\nabla g_\lambda\|\leq G\) with the baseline restriction
\(h\lesssim L_f^{-1}\), whereas the heat increment is controlled by
\(\tau_f+B_k\) rather than \(dL_\lambda\).
Transferring \(B_k\) to the stationary reference path yields the
fixed-\(\lambda\) complexity of
Theorem~\ref{thm:main-complete}, 
\begin{equation}
    N_\lambda(\varepsilon_{\mathrm{alg}})
    \lesssim
    \frac{1}{m}
    \left[
        L_f
        +
        \frac{
            \tau_f+G^2+B_{\mathrm{ref}}
        }{
            \varepsilon_{\mathrm{alg}}^2
        }
        +
        \frac{
            M_\lambda
        }{
            \varepsilon_{\mathrm{alg}}
        }
    \right],
    \label{eq:rw-our-fixed-lambda}
\end{equation}
up to logarithmic factors.

If no active-set information is used, the universal estimates
\[
    B_{\mathrm{ref}}
    \leq
    M_\lambda
    \leq
    \frac{d}{\lambda}
\]
and the choice
\(\lambda\asymp\varepsilon/G^2\) give
\begin{equation}
    \frac{B_{\mathrm{ref}}}{\varepsilon^2}
    \lesssim
    \frac{dG^2}{\varepsilon^3}.
    \label{eq:rw-our-global-trace-cubic}
\end{equation}
Thus the global-trace specialization of our theorem matches the
\(\varepsilon^{-3}\) dependence obtained by applying the Wasserstein--EVI bound to \(U_\lambda\), while resolving the smooth
contribution through \(\tau_f\) rather than \(dL_f\).

For the structured penalties studied in
Sections~\ref{sec:reftrace}--\ref{sec:examples-verification}, the
curvature--tube estimates instead give
\begin{equation}
    B_{\mathrm{ref}}
    \leq
    A_g,
    \label{eq:rw-structured-reference-trace}
\end{equation}
where \(A_g\) is independent of \(\lambda\).  If, in addition,
\[
    M_\lambda
    \leq
    \frac{r_g}{\lambda},
\]
then
\begin{equation}
    \frac{M_\lambda}{\varepsilon}
    \lesssim
    \frac{r_gG^2}{\varepsilon^2}
    \qquad
    \text{when }
    \lambda\asymp\frac{\varepsilon}{G^2}.
    \label{eq:rw-residual-active-trace-term}
\end{equation}
Substitution into
\eqref{eq:rw-our-fixed-lambda} therefore gives
\begin{equation}
    N_{\mathrm{active}}(\varepsilon)
    =
    \widetilde O\left[
        \frac{1}{m}
        \left\{
            L_f
            +
            \frac{
                \tau_f+A_g+(1+r_g)G^2
            }{
                \varepsilon^2
            }
        \right\}
    \right].
    \label{eq:rw-active-trace-quadratic}
\end{equation}

Suppressing dimensions, problem constants, and logarithms, the
resulting accuracy comparison is
\begin{equation}
\begin{array}{rcl}
\text{Wasserstein--EVI bound applied globally to \(U_\lambda\)}
&:&
\widetilde O(\varepsilon^{-3}),
\\[1mm]
\text{structured active-trace analysis of the same MYULA kernel}
&:&
\widetilde O(\varepsilon^{-2}).
\end{array}
    \label{eq:rw-direct-accuracy-comparison}
\end{equation}
The gain comes from replacing the global \(O(\lambda^{-1})\) curvature
charge in the leading error term by the occupation-weighted
\(B_{\mathrm{ref}}\), which is uniform in \(\lambda\) for the
structured examples.

\paragraph{Related proximal-MCMC developments.}
Subsequent work has developed the proximal-MCMC viewpoint in several
algorithmic directions. Moreau smoothing has been combined with
stabilized integrators and relaxed proximal-point iterations
\cite{pereyra2020accelerating,klatzer2024accelerated};
inexact-proximal analyses account for iterative or approximate evaluation
of the proximal map \cite{ehrhardt2024proximal}; and successive-Moreau
schemes vary the smoothing scale during sampling rather than fixing a
single \(\lambda\) \cite{habring2026diffusion}. These works modify the
integrator, the accuracy of the proximal computation, or the smoothing
schedule. The active-trace theory is complementary: it keeps the
classical fixed-\(\lambda\) MYULA kernel and sharpens the geometric
quantity that controls its discretization error.

\subsection{Average smoothness and trace-sensitive Langevin bounds}
\label{sec:related-average-smoothness}

The closest result in spirit is the average-smoothness theory of
\cite{dalalyan2026improved} for standard LMC on smooth strongly
log-concave targets.  Let $M_\infty$ denote the global Lipschitz constant of
the gradient, and let
\[
    M_{\mathrm{av}}=\frac{1}{d}\sum_{j=1}^d M_j
\]
be the average of coordinatewise smoothness constants.  When the potential is
twice differentiable, $M_\infty$ uniformly controls the largest Hessian
eigenvalue, whereas $M_{\mathrm{av}}$ controls a coordinatewise trace-type
quantity.  Their main constant-step estimate has the form
\[
    W_2^2(\nu_k,\pi)
    \leq e^{-2mkh}W_2^2(\nu_0,\pi)
       +\frac{(M_{\mathrm{av}}+m)hd}{2m},
    \qquad h\leq M_\infty^{-1}.
\]
Consequently, the leading discretization term depends on an average condition
number, although the maximal smoothness remains in the step-size restriction.

The active-trace result shares two features with this analysis.  Both replace
a maximal-curvature contribution in the leading error by a trace-like
quantity, and both can be interpreted through the isotropic Gaussian noise
in the Langevin update.  However, the two averaging operations are different.
The constants $M_j$ in average smoothness still take a supremum over all
spatial locations.  In particular, for a coordinatewise soft-threshold
Moreau envelope one still has $M_j=1/\lambda$ for every penalized coordinate,
regardless of how narrow the threshold region is.  By contrast,
\[
    B_{\mathrm{ref}}
    =\frac1h\int_0^h
       \mathbb E\,\operatorname{tr}H_\lambda(Y_t^{\mathrm{ref}})\,\mathrm dt
\]
averages over both directions and the spatial occupation measure of the
reference heat path.  It can therefore remain $O(1)$ even when
$\sup_x\operatorname{tr}H_\lambda(x)=O(d/\lambda)$.

There is also a regularity distinction.  Average-smoothness LMC is formulated
for a smooth potential using directional first-order inequalities or
classical Hessian quantities.  The Moreau envelope in the present work is
only $C^{1,1}$ in general.  Our trace is consequently defined through the
a.e./weak Hessian, and the heat energy identity is proved at this regularity.
Thus, one concise positioning statement is that the present framework is a
\emph{nonsmooth, distribution-weighted, active-set refinement of the
average-curvature principle}: average smoothness averages across directions
but remains worst-case over space, whereas active trace also averages over
where the discretized path goes.

Trace, Frobenius, and average-curvature quantities have appeared in other
refinements of Langevin discretization theory, especially when stronger
second- or third-order regularity is available.  Those results support the
broader message that operator-norm smoothness can be too coarse for sampling.
The specific feature here is the cancellation between a
$\lambda^{-1}$ curvature height and the $O(\lambda)$ occupation probability
of active tubes.

\subsection{Proximal splitting methods for the original composite target}
\label{sec:related-splitting}

A separate line of work treats the nonsmooth component through a proximal
splitting step rather than by first replacing the target with
$\pi_\lambda$.  The Wasserstein gradient-flow and convex-optimization
viewpoint of \cite{durmus2019analysis} gives a systematic
analysis of Langevin algorithms and includes nonsmooth proximal schemes.  In
the strongly convex composite setting, proximal stochastic gradient Langevin
algorithms and their primal--dual interpretation yield nonasymptotic
$W_2$ guarantees, including $O(\varepsilon^{-2})$ iteration complexity under
the assumptions of \cite{salim2019stochastic,salim2020primal}.
These methods can also accommodate extended-valued regularizers, such as
indicators of convex constraints, that fall outside the finite-valued
Lipschitz assumption used here.

This literature is crucial for delimiting our claim.  The present paper does
not give the first $O(\varepsilon^{-2})$ guarantee for sampling a nonsmooth
strongly log-concave distribution.  The distinction is algorithmic and
analytic.  A typical proximal splitting update applies a proximal map to a
noisy or forward-gradient proposal and is designed to approximate or preserve
the original composite target.  \MYULA{} instead runs an explicit Euler step
for the smooth surrogate $f+g_\lambda$ and incurs a separate smoothing bias.
Our question is therefore not whether a different proximal Langevin scheme
can attain a better generic rate, but whether the widely used \MYULA{}
iteration must pay the global $1/\lambda$ curvature at every step.  The
active-trace theorem shows that, for structured penalties, the leading
answer is no.

A useful way to present the relationship is by four axes: the transition
kernel, the reference target of one step, the oracle, and the error metric.
\MYULA{} and proximal splitting methods may both use one gradient and one
proximal evaluation per iteration, but the position of the prox operation and
the invariant or approximate target differ.  Accordingly, complexity
statements should not be ranked solely by their exponent in $\varepsilon$.
They should also report whether the guarantee concerns $\pi$ or
$\pi_\lambda$, whether smoothing bias is present, whether $g$ may be
extended-valued, and whether the result is in $W_2$, KL, or total variation.

\subsection{Active manifolds, proximal Jacobians, and degrees of freedom}
\label{sec:related-active-manifolds}

The geometric interpretation of $a_\lambda$ is related to the literature on
partial smoothness, active manifolds, and identification in nonsmooth
optimization.  Partly smooth regularizers behave smoothly along an active
manifold and sharply in normal directions; proximal and forward--backward
algorithms can identify this manifold and subsequently exhibit a lower
dimensional local dynamics \cite{lewis2002active,liang2017activity}.
For polyhedral and sparsity-promoting penalties, the Jacobian of the proximal
map is piecewise constant or admits an explicit tangent-space description.

Closely related formulas also arise in statistical degrees-of-freedom
calculations.  The divergence or trace of an estimator's Jacobian measures
the effective number of fitted degrees of freedom, and this program has been
developed for lasso, generalized lasso, analysis sparsity, group structure,
and more general partly smooth regularizers
\cite{tibshirani2012degrees,vaiter2017degrees}.
The identity
\[
    \lambda a_\lambda(x)
    =d-\operatorname{tr}D\prox_{\lambda g}(x)
\]
therefore has a natural interpretation: it is the complement of the local
proximal degrees of freedom, or the total number of directions locally
contracted by the regularizer.  In the generalized-lasso cells of
Subsection~\ref{subsec:generalized-lasso-example}, this becomes exactly
$\lambda a_\lambda(x)=\operatorname{rank}(D_A)$.

The proximal-Jacobian and active-manifold identities themselves are not the
novelty claimed here.  The new step is to insert this local codimension into a
Langevin discretization inequality and then average it over a stochastic
heat path.  The curvature--tube interface additionally quantifies how the
normal geometry of an active stratum interacts with its occupation
probability.  Thus, the paper connects two previously separate uses of
proximal geometry: local dimension and sensitivity in optimization/statistics,
and nonasymptotic discretization error in sampling.

\subsection{Alternative envelopes and smoothing-based Langevin methods}
\label{sec:related-envelopes}

Moreau smoothing is only one way to regularize a composite potential.  The
forward--backward envelope has been used to construct Langevin schemes that
retain additional optimization structure, including preservation of the MAP
point under suitable conditions \cite{eftekhari2023forward,ghaderi2024smoothing}.
Bregman--Moreau envelopes and Bregman proximal maps extend this idea to
non-Euclidean geometries and relative-smoothness settings
\cite{lau2022bregman}.  These methods alter the surrogate potential and
therefore alter the curvature object that enters the discretization analysis.
The reference active-trace idea suggests a possible extension: for another
envelope, one may seek the trace of its generalized local curvature and
average it over the associated one-step interpolation rather than bounding it
uniformly.

Other work focuses on numerical implementation rather than the choice of
envelope.  Stabilized explicit schemes permit larger stability domains for
stiff smoothed drifts \cite{pereyra2020accelerating}, relaxed
proximal-point algorithms modify the implicitness or acceleration of the
update \cite{klatzer2024accelerated}, and inexact-proximal methods quantify the
error caused by terminating an inner proximal solver
\cite{ehrhardt2024proximal}.  These concerns are orthogonal to
the present analysis, which assumes exact evaluations of
$\prox_{\lambda g}$ and asks how the exact Moreau curvature is weighted by the
path distribution.

\subsection{Direct nonsmooth sampling and stronger oracle models}
\label{sec:related-other-nonsmooth}

Several methods avoid a fixed Moreau approximation altogether.  Subgradient
Langevin schemes replace the gradient of the nonsmooth component by a
subgradient or a primal--dual construction and analyze the resulting
nonsmooth dynamics directly \cite{habring2024subgradient}.  Other
algorithms use proximal proposals inside a Metropolis--Hastings correction,
thereby targeting the original distribution exactly at stationarity
\cite{mou2022efficient}.  These methods answer a
different question from ours: whether one can sample the original nonsmooth
target without accepting a fixed regularization bias.

At the other end of the oracle spectrum, restricted-Gaussian-oracle methods
assume the ability to sample from densities proportional to
\[
    \exp\!\left\{-g(x)-\frac{\|x-y\|^2}{2\eta}\right\},
\]
which is a sampling analogue of a proximal evaluation.  Recent
proximal-gradient samplers achieve polylogarithmic dependence on the target
precision in strongly log-concave composite problems under this stronger
oracle \cite{liu2026proximal}.  Such results are important
benchmarks for composite sampling but are not directly comparable to
\MYULA{}, whose basic iteration uses a deterministic proximal map and one
Gaussian increment.  Any empirical or theoretical comparison should make the
oracle cost explicit.

In summary, the present paper occupies a specific position within this
landscape.  It does not replace the many algorithmic strategies for
nonsmooth sampling.  It refines the theory of one established strategy by
showing that the leading discretization cost can be governed by an
occupation-weighted, active-set measure of Moreau curvature.  The closest
analytic precedent is average-smoothness LMC; the closest algorithmic
alternatives are proximal splitting methods; and the closest geometric
precedents are active-manifold and degrees-of-freedom analyses of proximal
maps.  The active-trace framework combines these three viewpoints in a form
adapted to the stochastic heat step of \MYULA{}.

\section{Preliminaries}
\label{sec:preliminaries}

\subsection{Basic notation}
\label{subsec:basic-notation}

All probability measures are defined on $\R^d$ equipped with its Borel
$\sigma$-algebra.  The Euclidean norm and inner product are denoted by
$\norm{\cdot}$ and $\ip{\cdot}{\cdot}$.  For $\mu,\nu\in\cP_2(\R^d)$, the
quadratic Wasserstein distance is
\begin{equation}
   \W^2(\mu,\nu)
   := \inf_{(X,Y):\,\law(X)=\mu,\,\law(Y)=\nu} \E\norm{X-Y}^2 .
\end{equation}
If $\mu\ll \nu$, the relative entropy is
\begin{equation}
  \KL(\mu\Vert \nu) := \int \log\!\left(\frac{\dd\mu}{\dd\nu}\right)\dd\mu,
\end{equation}
and otherwise $\KL(\mu\Vert\nu):=+\infty$.  If
$\mu(\dd x)=\rho(x)\dd x$, its Lebesgue entropy is
\begin{equation}
   \Ent(\mu) := \int_{\R^d} \rho(x)\log \rho(x)\dd x,
\end{equation}
with the convention $\Ent(\mu)=+\infty$ if the expression is not well-defined.

The total variation distance is
\begin{equation}
   \norm{\mu-\nu}_{\TV}:=\sup_A |\mu(A)-\nu(A)|,
\end{equation}
where the supremum is over Borel sets.  If $\mu$ and $\nu$ have densities $p$
and $q$, then
\begin{equation}
   \norm{\mu-\nu}_{\TV}=\frac12\int_{\R^d}|p(x)-q(x)|\dd x.
\end{equation}

For a measurable map $T:\R^d\to\R^d$ and a probability measure $\mu$, the
pushforward $T_\#\mu$ is the law of $T(X)$ when $X\sim\mu$.  

Let $(P_t)_{t\ge 0}$ be the heat semigroup with generator $\Delta$:
\begin{equation}\label{eq:heat-semigroup}
   P_t\varphi(x) = \E\,\varphi(x+\sqrt{2t}Z), \qquad Z\sim N(0,\Id_d).
\end{equation}
Thus, if $\nu_t=\nu_0P_t$, then $\nu_t$ is the law of
$Y_t=Y_0+\sqrt{2t}Z$ for $Y_0\sim\nu_0$ independent of $Z$.

Throughout the paper, $C,c\in(0,\infty)$ denote universal constants whose value
may change from line to line.  They never depend on $d,\lambda,h,k$ or on the
particular measures under consideration, unless explicitly stated otherwise.

For symmetric matrices $A,B\in\R^{d\times d}$, we write $A\preceq B$ if
$v^\top Av\le v^\top Bv$ for every $v\in\R^d$.  This is the Loewner order, and
we write $\norm{A}_{\op}:=\sup_{\norm{v}=1}\norm{Av}$.

\subsection{Nonsmooth calculus for Lipschitz and convex functions}
\label{subsec:nonsmooth-calculus}

For ordinary notation, if $T:\R^{d_1}\to\R^{d_2}$ is differentiable at $x$, then
$DT(x)$ denotes its Jacobian matrix:
\begin{equation}
   DT(x)_{ij}=\frac{\partial T_i}{\partial x_j}(x).
\end{equation}

We use the global convention
\[
C^{1,1}(\mathbb R^d)
:=
\{V\in C^1(\mathbb R^d): \nabla V
\text{ is globally Lipschitz}\}.
\]
For $V\in C^{1,1}(\mathbb R^d)$, a standard theorem of Rademacher (reviewed in \Cref{app:rademacher}) says that \(\nabla V\) is differentiable a.e.; at those points we write
\[
   \nabla^2 V(x):=D(\nabla V)(x).
\]
This is the a.e. Hessian.  The same object can also be described as the
$L^\infty$ weak Hessian.  Here ``weak'' means defined through integration by
parts rather than by pointwise second derivatives, and $L^\infty$ means bounded
outside a null set.  A formal explanation is given in \Cref{app:weak-derivatives}.  In the main text, every Hessian of $g_\lambda$ is
understood in this a.e./weak sense, never as an everywhere classical $C^2$
Hessian.

If $g$ is convex, its subdifferential at $x$ is
\begin{equation}
   \partial g(x)
   :=\{s\in\R^d:\ g(y)\ge g(x)+\ip{s}{y-x}\text{ for all }y\in\R^d\}.
\end{equation}
The subdifferential is monotone: if $s\in\partial g(x)$ and
$t\in\partial g(y)$, then
\begin{equation}
   \ip{s-t}{x-y}\ge0.
\end{equation}

\subsection{Moreau--Yosida regularization and weak second-order structure}
\label{subsec:moreau-structure}

The Moreau envelope is the smoothing device used throughout the paper.  Let $g:\R^d\to(-\infty,+\infty]$ be proper, lower semicontinuous, and convex.
For $\lambda>0$, define the Moreau envelope and proximal map by
\begin{equation}
  g_\lambda(x)
  := \inf_{y\in\R^d}\left\{ g(y)+\frac{\norm{x-y}^2}{2\lambda}\right\},
  \qquad
  p_\lambda(x) := \prox_{\lambda g}(x)
  := \argmin_{y\in\R^d}\left\{ g(y)+\frac{\norm{x-y}^2}{2\lambda}\right\}.
\end{equation}
The following facts are standard in convex analysis and monotone operator
theory; see, for example, \cite[Chs.~12 and 23]{bauschke2020correction} and
\cite[Ch.~1.G]{rockafellar1998variational}. Detailed justification is given in \Cref{app:moreau-second-order}.

\begin{lemma}[Moreau regularity and a.e. Hessian]
\label{lem:moreau-regularity}
Let $g:\R^d\to(-\infty,+\infty]$ be proper, lower semicontinuous, and convex.
For every $\lambda>0$, the proximal map $p_\lambda$ is single-valued, and the
Moreau envelope $g_\lambda$ is convex and belongs to $C^{1,1}(\R^d)$.  Moreover,
\begin{equation}
  \nabla g_\lambda(x)=\frac{x-p_\lambda(x)}{\lambda},
  \qquad
  \Lip(\nabla g_\lambda)\le \lambda^{-1}.
\end{equation}
There exists a set \(E_\lambda\subset\mathbb R^d\), whose complement is
Lebesgue-null, such that the following statements hold for every
\(x\in E_\lambda\). The map \(\nabla g_\lambda\) is differentiable at \(x\);
writing
\[
\nabla^2 g_\lambda(x):=D(\nabla g_\lambda)(x),
\]
the matrix \(\nabla^2 g_\lambda(x)\) is symmetric and satisfies
\[
0\preceq \nabla^2 g_\lambda(x)\preceq \lambda^{-1}I_d.
\]
Furthermore, \(p_\lambda\) is differentiable at \(x\), and
\[
Dp_\lambda(x)
=
I_d-\lambda\nabla^2 g_\lambda(x),
\qquad
\nabla^2 g_\lambda(x)
=
\lambda^{-1}\{I_d-Dp_\lambda(x)\}.
\]

\end{lemma}

We next fix the curvature quantity that will enter all heat-flow estimates.
By \Cref{lem:moreau-regularity}, we choose a measurable
representative \(H_\lambda\) of the a.e./weak
Hessian of \(g_\lambda\) such that $H_\lambda(x)=\nabla^2g_\lambda(x)$ for a.e. $x$.

\begin{definition}[Active trace]
\label{def:active-trace}
The Moreau active trace is
\[
  a_\lambda(x):=\operatorname{tr}H_\lambda(x).
\]
Thus
\[
  0\le a_\lambda(x)\le d/\lambda
  \quad dx\text{-a.e.}
\]
By \Cref{lem:moreau-regularity},
\[
  a_\lambda(x)
  =
  \lambda^{-1}\operatorname{tr}\{I_d-Dp_\lambda(x)\}
  \quad dx\text{-a.e.}
\]
Hence \(a_\lambda\) is the trace of the local Moreau curvature. Equivalently,
up to the factor \(\lambda^{-1}\), it measures the total local shrinkage of
the proximal map.
\end{definition}

\begin{lemma}[Moreau envelope preserves Lipschitz constants]
\label{lem:moreau-gradient-bound}
Assume that $g:\R^d\to\R$ is convex and $G$-Lipschitz:
\begin{equation}
   |g(x)-g(y)|\le G\norm{x-y}, \qquad x,y\in\R^d .
\end{equation}
Then, for every $\lambda>0$ and every $x\in\R^d$,
\begin{equation}
   \norm{\nabla g_\lambda(x)}\le G .
\end{equation}
\end{lemma}

\subsection{The heat energy identity for \texorpdfstring{$C^{1,1}$}{C1,1} test functions}
\label{subsec:heat-energy}

We record a basic identity for Gaussian smoothing. Let $Y_t = Y_0+\sqrt{2}W_t$. For a smooth test function \(V\), It\^o's formula gives
\[
   \mathbb E V(Y_h)-\mathbb E V(Y_0)
   =
   \int_0^h \mathbb E\,\Delta V(Y_t)\,dt .
\]
Thus the expected change of \(V\) along the heat flow is controlled by
its Laplacian.

We will need this identity for functions that are only \(C^{1,1}\), since
Moreau envelopes have Lipschitz gradients but need not be twice
continuously differentiable. In this case the Hessian is interpreted in
the a.e./weak sense. The following lemma gives the
precise statement.

\begin{lemma}[Heat energy identity for \(C^{1,1}\) test functions]
\label{lem:heat-energy}
Let \(V\in C^{1,1}(\R^d)\). Choose any Lebesgue measurable representative
\(H_V\) of its a.e./weak Hessian and put
\[
   \Delta V := \tr H_V .
\]
Let \(W_t\) be a standard Brownian motion in \(\R^d\), independent of
\(Y_0\in L^2\), and set
\[
   Y_t = Y_0+\sqrt{2}W_t .
\]
Then, for every \(h>0\),
\[
  \E V(Y_h)-\E V(Y_0)
  =
  \int_0^h \E\,\Delta V(Y_t)\,dt .
\]
\end{lemma}

The proof is postponed to \Cref{app:heat-identity}. 
It is based on the It\^o--Krylov formula for functions with generalized second derivatives, together with a localization argument.

\subsection{Entropy EVI for the heat flow}
\label{subsec:heat-evi}

The next lemma is a standard finite-time inequality for the heat semigroup.
It will be used to control the entropy change during the Gaussian-noise part
of one MYULA step.

\begin{lemma}[Entropy EVI for the heat semigroup]\label{lem:heat-evi}
Let \(\nu_0\in\mathcal P_2(\mathbb R^d)\), and set
\[
\nu_t=\nu_0P_t,\qquad t\ge 0,
\]
where \(P_t\) is the heat semigroup defined in \eqref{eq:heat-semigroup}. Then, for every
\(\sigma\in\mathcal P_2(\mathbb R^d)\) with
\(\operatorname{Ent}(\sigma)<\infty\), and every \(h>0\),
\[
2h\{\operatorname{Ent}(\nu_h)-\operatorname{Ent}(\sigma)\}
\le
W_2^2(\nu_0,\sigma)-W_2^2(\nu_h,\sigma).
\]
\end{lemma}

\begin{proof}[Reference]
This is Lemma~5 of \cite{durmus2019analysis}. 
\end{proof}

\section{Problem Setup and Standing Assumptions}
\label{sec:setup}

We consider the nonsmooth composite target
\begin{equation}
\label{eq:original-target}
   \pi(\dd x) = Z^{-1}\exp\{-f(x)-g(x)\}\dd x,
\end{equation}
where $f$ is smooth and strongly convex and $g$ is convex but possibly
nonsmooth.  For a fixed Moreau parameter $\lambda>0$, define
\begin{equation}
\label{eq:smoothed-target}
   U_\lambda(x):=f(x)+g_\lambda(x),
   \qquad
   \pi_\lambda(\dd x) := Z_\lambda^{-1}\exp\{-U_\lambda(x)\}\dd x .
\end{equation}

\begin{assumption}[Smooth strongly convex part]
\label{ass:f}
The function $f\in C^2(\R^d)$ satisfies, for some $0<m\le L_f<\infty$,
\begin{equation}
   m\Id_d \preceq \nabla^2 f(x) \preceq L_f\Id_d,
   \qquad x\in\R^d.
\end{equation}
We also set
\begin{equation}
   \tau_f := \sup_{x\in\R^d}\tr\nabla^2 f(x) \le dL_f .
\end{equation}
\end{assumption}

\begin{assumption}[Convex Lipschitz nonsmooth part]
\label{ass:g}
The function $g:\R^d\to\R$ is closed, convex, and $G$-Lipschitz:
\begin{equation}
   |g(x)-g(y)|\le G\norm{x-y},
   \qquad x,y\in\R^d.
\end{equation}
\end{assumption}

Under \Cref{ass:f,ass:g}, the function $f$ is $C^2$ with bounded Hessian, while
$g_\lambda$ is $C^{1,1}$ by \Cref{lem:moreau-regularity}.  Hence
\begin{equation}
   U_\lambda=f+g_\lambda
\end{equation}
belongs to $C^{1,1}$ and has an a.e./weak Hessian in the sense of
\Cref{app:weak-derivatives}:
\begin{equation}
   H_{U_\lambda}(x)=\nabla^2 f(x)+H_\lambda(x)
\end{equation}
for a.e. $x$. Since $f$ is $m$-strongly convex and
$g_\lambda$ is convex, $U_\lambda$ is $m$-strongly convex.  Therefore
$\pi_\lambda$ is well-defined and has finite second moment.

The MYULA transition with step size $h>0$ is the explicit Euler transition for
$\pi_\lambda$:
\begin{equation}
\label{eq:myula-update}
   X_{k+1}=X_k-h\nabla U_\lambda(X_k)+\sqrt{2h}\,\xi_{k+1},
   \qquad \xi_{k+1}\sim N(0,\Id_d).
\end{equation}
We analyze the error to the fixed smoothed target $\pi_\lambda$ first.  The
regularization bias $\W(\pi_\lambda,\pi)$ is separated from the fixed-$\lambda$
discretization analysis.

Let $\mu_k:=\law(X_k)$ and define
\begin{equation}
\label{eq:D-K-def}
   D_k:=\W^2(\mu_k,\pi_\lambda),
   \qquad
   K_k:=\KL(\mu_k\Vert\pi_\lambda).
\end{equation}
Throughout the fixed-\(\lambda\) analysis we assume $\mu_0\in\cP_2(\R^d)$ and $K_0<\infty$.

It is useful to split one MYULA step into a deterministic gradient step followed
by a heat step.  Define
\begin{equation}
\label{eq:Th}
   T_h(x):=x-h\nabla U_\lambda(x),
   \qquad
   \bar\mu_k:=(T_h)_\#\mu_k .
\end{equation}
Then $\mu_{k+1}=\bar\mu_kP_h$.  Let
\((W^{(k)}_t)_{0\le t\le h}\) be a standard Brownian motion in
\(\mathbb R^d\), independent of \(X_k\), and define the heat interpolation
\[
Y_{k,t}:=T_h(X_k)+\sqrt{2}\,W^{(k)}_t,\qquad 0\le t\le h .
\]

\begin{definition}[Stepwise heat-path active trace]
\label{def:Bk}
Let \(a_\lambda=\operatorname{tr}H_\lambda\) be the active-trace
representative fixed in \Cref{def:active-trace}.  For the heat interpolation
\((Y_{k,t})_{0\le t\le h}\), define
\begin{equation}
\label{eq:Bk}
   B_k:=\frac{1}{h}\int_0^h \E\,a_\lambda(Y_{k,t})\dd t.
\end{equation}
The time integral may equivalently be read as an integral over $(0,h]$, since
changing the integrand at the single time $t=0$ has no effect.  For every
$t>0$, the law of $Y_{k,t}$ has a density, so $B_k$ is independent of the values
assigned to $H_\lambda$ on null sets.  The trivial global bound is
$B_k\le d/\lambda$.
\end{definition}

\section{Main Result: Reference Active Trace and Total Error}
\label{sec:main-result}

To state the complete error guarantee, we first introduce a reference
version of the heat-path active trace. Unlike \(B_k\), which is computed
from the current law \(\mu_k\), this quantity is computed from the reference target law \(\pi_\lambda\).

\begin{definition}[Reference heat-path active trace]
\label{def:reference-active-trace}
Let \(X^{\rm ref}\sim\pi_\lambda\), and let
\((W^{\rm ref}_t)_{0\le t\le h}\) be a standard Brownian motion in
\(\mathbb R^d\), independent of \(X^{\rm ref}\). Define
\[
Y^{\rm ref}_t
:=
T_h(X^{\rm ref})+\sqrt{2}\,W^{\rm ref}_t,
\qquad 0\le t\le h .
\]
The reference active trace is
\[
B_{\rm ref}
:=
\frac1h\int_0^h \mathbb E\,a_\lambda(Y^{\rm ref}_t)\,dt .
\]
Equivalently, \(B_{\rm ref}\) is the same heat-path average as \(B_k\),
with the initialization \(X_k\sim\mu_k\) replaced by
\(X^{\rm ref}\sim\pi_\lambda\). This quantity depends on \(\lambda\)
and \(h\), but this dependence is suppressed in the notation. As with
\(B_k\), the integral may be read over \((0,h]\), so the value is
independent of the chosen null-set version of \(a_\lambda\).
\end{definition}

The main theorem below expresses the fixed-\(\lambda\) discretization
error in terms of \(B_{\rm ref}\), and then adds the Moreau approximation
bias to control the error to the original target \(\pi\).

\begin{theorem}[Main active-trace guarantee]
\label{thm:main-complete}
Assume \Cref{ass:f,ass:g}.  Fix $\lambda>0$ and $h>0$.  Let \(M_\lambda\) be any finite constant such that
\(0\le a_\lambda\le M_\lambda\) Lebesgue-a.e. 
The universal choice \(M_\lambda=d/\lambda\) is always admissible. There exist universal constants $c,C>0$ such that, whenever $0<h\le c/L_f$, the MYULA
iterates satisfy
\begin{equation}
\label{eq:main-fixed-lambda-bound}
   \Phi_N
   \le e^{-mhN/2}\Phi_0
      +\frac{C}{m}\Bigl\{(\tau_f+G^2+B_{\mathrm{ref}})h
      +M_\lambda^2h^2\Bigr\},
   \qquad
   \Phi_k:=D_k+2hK_k .
\end{equation}
Consequently, for any algorithmic tolerance $\eps_{\mathrm{alg}}\in(0,1)$, if
\begin{equation}
\label{eq:main-h-choice}
   h\le c\min\left\{
      L_f^{-1},\,
      \frac{\eps_{\mathrm{alg}}^2}{\tau_f+G^2+B_{\mathrm{ref}}},\,
      \frac{\eps_{\mathrm{alg}}}{M_\lambda}
   \right\}
\end{equation}
and
\begin{equation}
\label{eq:main-N-choice}
   N\ge \frac{C}{mh}
      \log\left(1+\frac{m\Phi_0}{\eps_{\mathrm{alg}}^2}\right),
\end{equation}
then
\begin{equation}
\label{eq:main-fixed-target-error}
   \sqrt m\,\W(\mu_N,\pi_\lambda)\le \eps_{\mathrm{alg}}.
\end{equation}
Moreover, the original nonsmooth target $\pi$ satisfies the Moreau bias bound
\begin{equation}
\label{eq:main-bias-condition}
   \sqrt m\,\W(\pi_\lambda,\pi)\le \frac{G^2\lambda}{4}.
\end{equation}
Therefore, if $\eps_{\mathrm{bias}}>0$ and
\begin{equation}
   \lambda\le \frac{4\eps_{\mathrm{bias}}}{G^2},
\end{equation}
then the total error obeys
\begin{equation}
\label{eq:main-total-error}
   \sqrt m\,\W(\mu_N,\pi)
   \le \eps_{\mathrm{alg}}+\eps_{\mathrm{bias}}.
\end{equation}
In the common symmetric choice
$\eps_{\mathrm{alg}}=\eps_{\mathrm{bias}}=\eps/2$, it is enough to take
$\lambda\le 2\eps/G^2$, and the sufficient iteration complexity is, up to
logarithmic factors,
\begin{equation}
\label{eq:main-complexity-informal}
   N(\eps)
   \lesssim
   \frac1m\left(
       L_f+\frac{\tau_f+G^2+B_{\mathrm{ref}}}{\eps^2}
       +\frac{M_\lambda}{\eps}
   \right).
\end{equation}
\end{theorem}

\begin{remark}
The term $B_{\mathrm{ref}}$ is the main active-geometry quantity.  If only the
universal bound $a_\lambda\le d/\lambda$ is available, then
$B_{\mathrm{ref}}\le d/\lambda$ and the theorem reduces to a conservative
worst-case result.  The purpose of \Cref{sec:reftrace} is to prove
that $B_{\mathrm{ref}}$ is much smaller than $d/\lambda$ in structured examples.
\end{remark}

The proof of \Cref{thm:main-complete} is given in \Cref{subsec:proof-main} after
all three ingredients have been established.

\subsection{Universal Active-Trace KL--EVI Recursion}
\label{subsec:main-recursion}

We first derive a one-step recursion for the fixed-$\lambda$ target $\pi_\lambda$. 
The argument separates one MYULA update into a deterministic gradient step and a heat step. 
The gradient step is controlled by an approximate EVI estimate whose error depends on the Lipschitz size $G^2$ of $g$, rather than on the Moreau smoothness scale $\lambda^{-2}$. 
Combining this estimate with the entropy EVI for the heat step yields a recursion involving the stepwise active trace $B_k$.

\begin{lemma}[Gradient-step approximate EVI without $1/\lambda$]
\label{lem:gradient-evi}
Assume \Cref{ass:f,ass:g}.  There exist universal constants $c_0,C_0>0$ such
that, if $0<h\le c_0/L_f$, then for all $x,y\in\R^d$, with $x^+=T_h(x)$,
\begin{equation}
\label{eq:gradient-evi}
   2h\{U_\lambda(x^+)-U_\lambda(y)\}
   \le
   \norm{x-y}^2-\norm{x^+-y}^2
   -mh\norm{x-y}^2
   +C_0G^2h^2 .
\end{equation}
\end{lemma}

\begin{proof}
Write
\begin{equation}
   s:=\nabla f(x), \qquad r:=\nabla g_\lambda(x), \qquad
   v:=s+r=\nabla U_\lambda(x), \qquad a:=x-y.
\end{equation}
Then $x^+=x-hv$.  By the strong convexity and smoothness of $f$,
\begin{equation}
\label{eq:f-step-bound}
   f(x^+)-f(y) = f(x)-f(y)+f(x^+)-f(x)
   \le \ip{s}{a}-\frac{m}{2}\norm{a}^2
      -h\ip{s}{v}+\frac{L_fh^2}{2}\norm{v}^2 .
\end{equation}
By convexity of $g_\lambda$ and by \Cref{lem:moreau-gradient-bound},
$g_\lambda$ is $G$-Lipschitz; hence
\begin{equation}
\label{eq:g-step-bound}
   g_\lambda(x^+)-g_\lambda(y)=g_\lambda(x)-g_\lambda(y)+g_\lambda(x^+)-g_\lambda(x)
   \le \ip{r}{a}+Gh\norm{v} .
\end{equation}
Adding \eqref{eq:f-step-bound} and \eqref{eq:g-step-bound}, multiplying by
$2h$, and using
\begin{equation}
   2h\ip{v}{a}=\norm{a}^2-\norm{a-hv}^2+h^2\norm{v}^2,
\end{equation}
we obtain
\begin{align}
  2h\{U_\lambda(x^+)-U_\lambda(y)\}
  &\le \norm{a}^2-\norm{a-hv}^2-mh\norm{a}^2  \notag\\
  &\quad +h^2\norm{v}^2-2h^2\ip{s}{v}
       +L_fh^3\norm{v}^2+2Gh^2\norm{v}.        \label{eq:pre-cancel}
\end{align}
The key cancellation is
\begin{equation}
   \norm{v}^2-2\ip{s}{v}
   =\norm{s+r}^2-2\ip{s}{s+r}
   =\norm{r}^2-\norm{s}^2
   \le G^2-\norm{s}^2.
\end{equation}
Moreover, $\norm{v}\le\norm{s}+G$.  If $h\le c_0/L_f$ with $c_0>0$ sufficiently
small, then
\begin{equation}
   L_fh^3\norm{v}^2+2Gh^2\norm{v}
   \le \frac12h^2\norm{s}^2+CG^2h^2 .
\end{equation}
Substituting this bound into \eqref{eq:pre-cancel} absorbs the negative
$-h^2\norm{s}^2$ term and gives \eqref{eq:gradient-evi}.
\end{proof}

\begin{theorem}[Universal active-trace KL--EVI recursion]
\label{thm:active-trace-recursion}
Assume \Cref{ass:f,ass:g}.  There exist universal constants $c,C>0$ such that,
for every $\lambda>0$ and every step size $0<h\le c/L_f$, the MYULA iterates
satisfy
\begin{equation}
\label{eq:active-trace-recursion}
   D_{k+1}+2hK_{k+1}
   \le (1-mh)D_k+C\bigl(\tau_f+G^2+B_k\bigr)h^2 .
\end{equation}
\end{theorem}

\begin{proof}
Let $(X,Y)$ be an optimal $\W$-coupling of $\mu_k$ and $\pi_\lambda$.  Apply
\Cref{lem:gradient-evi} pointwise to $(X,Y)$ and integrate.  Since
$(T_h(X),Y)$ is a coupling of $\bar\mu_k$ and $\pi_\lambda$, we have
\begin{equation}
\label{eq:integrated-gradient-step}
  2h\left\{\int U_\lambda\dd\bar\mu_k-
             \int U_\lambda\dd\pi_\lambda\right\}
  \le D_k-\W^2(\bar\mu_k,\pi_\lambda)-mhD_k+CG^2h^2 .
\end{equation}
Next, apply \Cref{lem:heat-evi} with
$\nu_0=\bar\mu_k$, $\nu_h=\mu_{k+1}$, and $\sigma=\pi_\lambda$:
\begin{equation}
\label{eq:heat-evi-step}
   2h\{\Ent(\mu_{k+1})-\Ent(\pi_\lambda)\}
   \le \W^2(\bar\mu_k,\pi_\lambda)-D_{k+1} .
\end{equation}

By \Cref{lem:heat-energy} applied to the $C^{1,1}$ function
$U_\lambda=f+g_\lambda$, with weak/a.e. Hessian
\begin{equation}
   H_{U_\lambda}=\nabla^2f+H_\lambda,
\end{equation}
we have
\begin{equation}
\label{eq:heat-energy-step}
   \int U_\lambda\dd\mu_{k+1}-\int U_\lambda\dd\bar\mu_k
   = \int_0^h \E\,\Delta U_\lambda(Y_{k,t})\dd t .
\end{equation}
Here
\begin{equation}
   \Delta U_\lambda=\tr\nabla^2f+\tr H_\lambda
\end{equation}
in the weak/a.e. sense explained in \Cref{app:weak-derivatives}.  Since $\tr\nabla^2f\le\tau_f$ pointwise and
$a_\lambda=\tr H_\lambda$, the right-hand side of
\eqref{eq:heat-energy-step} is bounded by
\begin{equation}
   h\tau_f+hB_k .
\end{equation}

Because
\begin{equation}
   K_{k+1}=\Ent(\mu_{k+1})-\Ent(\pi_\lambda)
   +\int U_\lambda\dd\mu_{k+1}
   -\int U_\lambda\dd\pi_\lambda,
\end{equation}
summing \eqref{eq:integrated-gradient-step}, \eqref{eq:heat-evi-step}, and
$2h$ times \eqref{eq:heat-energy-step} gives
\begin{equation}
   2hK_{k+1}
   \le D_k-D_{k+1}-mhD_k+CG^2h^2+2(\tau_f+B_k)h^2 .
\end{equation}
Rearranging yields \eqref{eq:active-trace-recursion}.
\end{proof}

\subsection{Closing the Recursion by Active-Trace Transfer}
\label{subsec:transfer}

The previous theorem reduces the fixed-$\lambda$ analysis to a bound on $B_k$. Now we bound $B_k$ by comparing the current heat path
to the reference heat path introduced in \Cref{sec:main-result}.
For \(0\le t\le h\), let \(Q_t\) be the Markov kernel obtained by
applying \(T_h\) and then running the heat semigroup for time \(t\):
\[
Q_t(x,\cdot):=(\delta_{T_h(x)}P_t)(\cdot).
\]
Equivalently, \(Q_t(x,\cdot)\) is the law of
\(T_h(x)+\sqrt{2t}Z\), where \(Z\sim N(0,I_d)\). Hence
\[
\mathcal L(Y_{k,t})=\mu_k Q_t,
\qquad
\mathcal L(Y_t^{\rm ref})=\pi_\lambda Q_t.
\]
Therefore,
\[
B_k
=
\frac1h\int_0^h \int a_\lambda\,\dd(\mu_kQ_t)\,dt,
\qquad
B_{\rm ref}
=
\frac1h\int_0^h \int a_\lambda\,\dd(\pi_\lambda Q_t)\,dt.
\]

\begin{proposition}[Bounding \(B_k\) by the reference active trace]
\label{prop:reference-TV-transfer}
Assume that $a_\lambda=\tr H_\lambda$ satisfies
$0\le a_\lambda\le M_\lambda$ a.e. Then, for every $k$,
\begin{equation}
\label{eq:Bk-Bref-direct}
   B_k\le B_{\mathrm{ref}}+M_\lambda\sqrt{K_k/2}.
\end{equation}
\end{proposition}

\begin{proof}
For every $t>0$, the measures $\mu_kQ_t$ and
$\pi_\lambda Q_t$ are absolutely continuous with respect to Lebesgue measure,
since $Q_t$ adds a nondegenerate Gaussian noise.  Hence the a.e. bound
$0\le a_\lambda\le M_\lambda$ may be used under both measures.  Therefore,
\begin{align}
   \int a_\lambda\dd(\mu_kQ_t)
   &\le \int a_\lambda\dd(\pi_\lambda Q_t)
       +M_\lambda\norm{\mu_kQ_t-\pi_\lambda Q_t}_{\TV} \\
   &\le \int a_\lambda\dd(\pi_\lambda Q_t)
       +M_\lambda\norm{\mu_k-\pi_\lambda}_{\TV} .
\end{align}
The second inequality is contraction of total variation under Markov kernels.
Pinsker's inequality gives
\begin{equation}
   \norm{\mu_k-\pi_\lambda}_{\TV}
   \le \sqrt{\KL(\mu_k\Vert\pi_\lambda)/2}
   = \sqrt{K_k/2}.
\end{equation}
Integrating the
preceding display over $t\in(0,h]$ and dividing by $h$ gives the same value as
integrating over $[0,h]$.  Hence \eqref{eq:Bk-Bref-direct} follows.
\end{proof}

Substituting this transfer estimate into the universal recursion gives a closed Lyapunov recursion in terms of $B_{\rm ref}$ only.

\begin{theorem}
\label{thm:abstract-complexity}
Assume \Cref{ass:f}, \Cref{ass:g} and the assumption of \Cref{prop:reference-TV-transfer}.  Let
\begin{equation}
   \Phi_k:=D_k+2hK_k .
\end{equation}
There exist universal constants $c,C>0$ such that, whenever $0<h\le c/L_f$,
\begin{equation}
\label{eq:Phi-recursion}
   \Phi_{k+1}
   \le \left(1-\frac{mh}{2}\right)\Phi_k
      +C\bigl\{(\tau_f+G^2+B_{\mathrm{ref}})h^2+M_\lambda^2h^3\bigr\} .
\end{equation}
Consequently,
\begin{equation}
\label{eq:Phi-global}
   \Phi_N
   \le e^{-mhN/2}\Phi_0
      +\frac{C}{m}\bigl\{(\tau_f+G^2+B_{\mathrm{ref}})h+M_\lambda^2h^2\bigr\} .
\end{equation}
In particular, given $\eps_{\mathrm{alg}}\in(0,1)$, if
\begin{equation}
\label{eq:stepsize-complexity-choice}
   h\le c\min\left\{
        L_f^{-1},\,
        \frac{\eps_{\mathrm{alg}}^2}{\tau_f+G^2+B_{\mathrm{ref}}},\,
        \frac{\eps_{\mathrm{alg}}}{M_\lambda}
      \right\},
\end{equation}
and
\begin{equation}
\label{eq:N-choice}
   N\ge \frac{C}{mh}\log\left(1+\frac{m\Phi_0}{\eps_{\mathrm{alg}}^2}\right),
\end{equation}
then
\begin{equation}
   \sqrt{m}\,\W(\mu_N,\pi_\lambda)\le \eps_{\mathrm{alg}} .
\end{equation}
Equivalently, up to logarithmic factors, the fixed-$\lambda$ iteration
complexity is
\begin{equation}
\label{eq:complexity-summary}
   N_\lambda(\eps_{\mathrm{alg}})
   \lesssim
   \frac{1}{m}\left(
       L_f+\frac{\tau_f+G^2+B_{\mathrm{ref}}}{\eps_{\mathrm{alg}}^2}
       +\frac{M_\lambda}{\eps_{\mathrm{alg}}}
   \right).
\end{equation}
\end{theorem}

\begin{proof}
By \Cref{thm:active-trace-recursion} and Proposition~\ref{prop:reference-TV-transfer},
\begin{equation}
   \Phi_{k+1}
   \le (1-mh)D_k+C(\tau_f+G^2+B_{\mathrm{ref}})h^2+CM_\lambda h^2\sqrt{K_k} .
\end{equation}
Young's inequality gives
\begin{equation}
   CM_\lambda h^2\sqrt{K_k}
   \le hK_k+CM_\lambda^2h^3 .
\end{equation}
If $mh\le1$, then
\begin{equation}
   (1-mh)D_k+hK_k
   \le \left(1-\frac{mh}{2}\right)(D_k+2hK_k)
   = \left(1-\frac{mh}{2}\right)\Phi_k .
\end{equation}
This proves \eqref{eq:Phi-recursion}.  Iterating the affine recursion yields
\eqref{eq:Phi-global}.  The choices \eqref{eq:stepsize-complexity-choice} and
\eqref{eq:N-choice} ensure that $m\Phi_N\le \eps_{\mathrm{alg}}^2$, while
$D_N\le \Phi_N$, so the Wasserstein conclusion follows.
\end{proof}

\subsection{Moreau Approximation Bias}
\label{subsec:bias}

It remains to relate the smoothed target $\pi_\lambda$ to the original nonsmooth target $\pi$. 
The next estimates give a universal bound on the bias introduced by replacing $g$ with its Moreau envelope $g_\lambda$.

Let
\begin{equation}
   \delta_\lambda(x):=g(x)-g_\lambda(x) .
\end{equation}
By choosing $y=x$ in the Moreau envelope, $\delta_\lambda\ge0$.  If $g$ is
$G$-Lipschitz, then
\begin{equation}
\label{eq:delta-bound}
   0\le \delta_\lambda(x)
   \le \frac{G^2\lambda}{2} .
\end{equation}
Indeed, $g(y)\ge g(x)-G\norm{x-y}$, so
\begin{equation}
   g_\lambda(x)
   \ge g(x)+\inf_{r\ge0}\left\{-Gr+\frac{r^2}{2\lambda}\right\}
   =g(x)-\frac{G^2\lambda}{2}.
\end{equation}
Moreover,
\begin{equation}
\label{eq:tilt-pilambda-pi}
   \frac{\dd\pi_\lambda}{\dd\pi}(x)
   =\frac{\exp\{\delta_\lambda(x)\}}
          {\E_\pi\exp\{\delta_\lambda(X)\}} .
\end{equation}

\begin{lemma}[Bounded exponential tilt]
\label{lem:bounded-tilt}
Let $P$ be a probability measure and let $0\le \delta\le a$.  Define $Q$ by
\begin{equation}
   \frac{\dd Q}{\dd P}=\frac{e^\delta}{\E_P e^\delta} .
\end{equation}
Then
\begin{equation}
   \KL(Q\Vert P)\le \frac{a^2}{8} .
\end{equation}
\end{lemma}

\begin{proof}
Let $K(t):=\log\E_Pe^{t\delta}$.  Then $K''(t)=\Var_{P_t}(\delta)$, where
$\dd P_t=e^{t\delta-K(t)}\dd P$.  Since $\delta\in[0,a]$,
$\Var_{P_t}(\delta)\le a^2/4$.  Therefore
\begin{equation}
   \KL(Q\Vert P)=K'(1)-K(1)=\int_0^1 tK''(t)\dd t\le \frac{a^2}{8}.
\end{equation}
\end{proof}

\begin{proposition}[Universal Moreau bias]
\label{prop:general-bias}
Under \Cref{ass:f,ass:g},
\begin{equation}
\label{eq:KL-bias}
   \KL(\pi_\lambda\Vert\pi)
   \le \frac{G^4\lambda^2}{32} .
\end{equation}
Consequently,
\[
\sqrt m\, W_2(\pi_\lambda,\pi)
\le
\frac{G^2\lambda}{4}.
\]
In particular, to make the Moreau bias at most \(\varepsilon_{\rm bias}\) in
\(\sqrt m W_2\), it is sufficient to choose
\[
\lambda \le \frac{4\varepsilon_{\rm bias}}{G^2}.
\]

\end{proposition}

\begin{proof}
By the bound \(0\le \delta_\lambda\le G^2\lambda/2\) and the density-ratio
identity
\[
\frac{\dd\pi_\lambda}{\dd\pi}(x)
=
\frac{\exp\{\delta_\lambda(x)\}}
{\mathbb E_\pi \exp\{\delta_\lambda(X)\}},
\]
\Cref{lem:bounded-tilt}, applied with \(a=G^2\lambda/2\), gives
\[
\KL(\pi_\lambda\|\pi)
\le
\frac{1}{8}\left(\frac{G^2\lambda}{2}\right)^2
=
\frac{G^4\lambda^2}{32}.
\]
Since \(f+g\) is \(m\)-strongly convex, \(\pi\) is \(m\)-strongly log-concave.
The Talagrand \(T_2\) inequality \cite{otto2000generalization,villani2009optimal} therefore yields
\[
W_2^2(\pi_\lambda,\pi)
\le
\frac{2}{m}\KL(\pi_\lambda\|\pi)
\le
\frac{G^4\lambda^2}{16m}.
\]
Taking square roots gives the claimed Wasserstein bound.

\end{proof}

\subsection{Proof of the main theorem}
\label{subsec:proof-main}

\begin{proof}[Proof of \Cref{thm:main-complete}]
\Cref{thm:abstract-complexity} gives
\[
   \Phi_N
   \le e^{-mhN/2}\Phi_0
      +\frac{C}{m}\Bigl\{(\tau_f+G^2+B_{\mathrm{ref}})h
      +M_\lambda^2h^2\Bigr\}.
\]
This is
\eqref{eq:main-fixed-lambda-bound}.  The step-size and iteration choices in
\eqref{eq:main-h-choice}--\eqref{eq:main-N-choice} make the right-hand side at
most $\eps_{\mathrm{alg}}^2/m$, and therefore
$\sqrt m\,\W(\mu_N,\pi_\lambda)\le\eps_{\mathrm{alg}}$.

The bias estimate \eqref{eq:main-bias-condition} is exactly
\Cref{prop:general-bias}.  If $\lambda\le4\eps_{\mathrm{bias}}/G^2$, then
$\sqrt m\,\W(\pi_\lambda,\pi)\le\eps_{\mathrm{bias}}$.  The triangle inequality
for $W_2$ gives
\[
   \sqrt m\,\W(\mu_N,\pi)
   \le \sqrt m\,\W(\mu_N,\pi_\lambda)
      +\sqrt m\,\W(\pi_\lambda,\pi)
   \le \eps_{\mathrm{alg}}+\eps_{\mathrm{bias}}.
\]
The displayed complexity follows by substituting
$\eps_{\mathrm{alg}}=\eps_{\mathrm{bias}}=\eps/2$ and suppressing universal
constants and logarithmic factors.
\end{proof}

\section{Bounding the Reference Active Trace}
\label{sec:reftrace}

The main theorem reduces the fixed-\(\lambda\) analysis to the reference active trace

\[
B_{\rm ref}
:=
\frac1h\int_0^h \mathbb E a_\lambda(Y_t^{\rm ref})\,\dd t,
\qquad
Y_t^{\rm ref}
:=
T_h(X^{\rm ref})+\sqrt{2}\,W_t^{\rm ref},
\qquad
X^{\rm ref}\sim \pi_\lambda .
\]
Here \(W^{\rm ref}\) is a standard Brownian motion in \(\mathbb R^d\), independent of
\(X^{\rm ref}\), \(T_h(x)=x-h\nabla U_\lambda(x)\), and \(a_\lambda=\operatorname{tr}H_\lambda\).
Equivalently, for each fixed \(t\), \(Y_t^{\rm ref}\) has the same law as
\(T_h(X^{\rm ref})+\sqrt{2t}Z\), where \(Z\sim N(0,I_d)\).

The purpose of this section is to give practical bounds on \(B_{\rm ref}\) from the geometry of
the nonsmooth penalty \(g\). The basic mechanism is that the Moreau curvature may be of order
\(1/\lambda\), but in many structured examples it is concentrated in an \(O(\lambda)\)-neighborhood
of an active set. If the reference heat path assigns probability of order \(\lambda\) to such
neighborhoods, then the factor \(1/\lambda\) is cancelled.

We first formulate this idea as an abstract curvature--tube interface. We then prove slice-density
and propagation estimates for the reference heat path, which provide the tube bounds used below.

\subsection{An abstract curvature-tube interface}

The first result is only an interface: it says that a curvature decomposition plus a tube-mass bound implies a bound on \(B_{\rm ref}\). The following subsections
provide general density estimates that make these tube bounds easy to verify.

\begin{assumption}[Curvature decomposition]
\label{ass:reftrace-curvature-decomposition}
For the fixed \(\lambda>0\), suppose that there are a nonnegative measurable function \(b_\lambda\), closed sets \(\Sigma_1,\ldots,\Sigma_J\subset \R^d\), and constants \(a_j,c_j>0\) such that, for Lebesgue-a.e. \(x\),
\begin{equation}
\label{eq:reftrace-curvature-decomposition}
 a_\lambda(x)
 \le b_\lambda(x)
 + \sum_{j=1}^J \frac{c_j}{\lambda}
 \mathbf 1\{\dist(x,\Sigma_j)\le a_j\lambda\}.
\end{equation}
\end{assumption}
Here \(b_\lambda\) represents the part of the curvature that is already integrable along the reference path.  The sets \(\Sigma_j\) represent active or singular sets where the Moreau curvature may be large.

\begin{assumption}[Reference tube mass]
\label{ass:reftrace-tube-mass}
For the reference path \((Y_t^{\rm ref})_{0\le t\le h}\), suppose that there are constants \(B_0,S_j,r_0>0\) such that
\begin{equation}
\label{eq:reftrace-blambda-average}
\frac{1}{h}\int_0^h \E b_\lambda(Y_t^{\rm ref})\,\dd t \le B_0,
\end{equation}
and, for every \(j\in\{1,\ldots,J\}\) and every \(0<r\le r_0\),
\begin{equation}
\label{eq:reftrace-tube-mass}
\frac{1}{h}\int_0^h
\Pp\{\dist(Y_t^{\rm ref},\Sigma_j)\le r\}\,\dd t
\le S_j r^{q_j}.
\end{equation}
\end{assumption}
The exponent \(q_j\) describes the tube-mass scaling near \(\Sigma_j\). For instance,
a codimension-one slab has exponent one, while a ball in a \(q\)-dimensional block has exponent
\(q\).

\begin{proposition}[Curvature-tube control of \(B_{\rm ref}\)]
\label{prop:reftrace-tube-control}
Assume \Cref{ass:reftrace-curvature-decomposition,ass:reftrace-tube-mass}.  If \(a_j\lambda\le r_0\) for every \(j\), then
\begin{equation}
\label{eq:reftrace-tube-bound}
B_{\rm ref}
\le B_0+
\sum_{j=1}^J c_j a_j^{q_j}S_j\lambda^{q_j-1}.
\end{equation}
In particular, every layer with \(q_j\ge 1\) produces no \(1/\lambda\) blow-up in the reference active trace.
\end{proposition}

\begin{proof}
For every \(t>0\), the law of \(Y_t^{\rm ref}\) has a density.  Hence the a.e. curvature inequality \eqref{eq:reftrace-curvature-decomposition} may be integrated along the path; the single endpoint \(t=0\) has no effect on the time average.  Therefore
\begin{align}
B_{\rm ref}
&\le \frac{1}{h}\int_0^h \E b_\lambda(Y_t^{\rm ref})\,\dd t
 + \sum_{j=1}^J \frac{c_j}{\lambda h}
 \int_0^h
 \Pp\{\dist(Y_t^{\rm ref},\Sigma_j)\le a_j\lambda\}\,\dd t  \notag\\
&\le B_0+
\sum_{j=1}^J \frac{c_j}{\lambda} S_j(a_j\lambda)^{q_j}.
\end{align}
\end{proof}

\begin{remark}[Role of the interface]
\Cref{prop:reftrace-tube-control} is a bookkeeping device rather than a final
assumption. The rest of this section gives sufficient conditions under which the tube-mass constants
\(S_j\) and exponents \(q_j\) follow from slice-density bounds for \(\pi_\lambda\) and their
propagation under \(T_h\) and the heat step.
\end{remark}

\subsection{Slice density of \(\pi_\lambda\)}

We now develop the density estimates used to control slab and ball
probabilities for the reference heat path. The starting point is a
slice-density bound for the stationary law \(\pi_\lambda\). The key
feature is that, along a chosen subspace \(E\), the bound is governed by
directional quantities for \(f\) and \(g\), rather than by the global
Moreau smoothness scale \(\lambda^{-1}\).

Let \(E\subset\R^d\) be a fixed linear subspace with dimension \(q\ge 1\).  Let \(P_E\) denote the orthogonal projection onto \(E\), and let \(E^\perp\) be the orthogonal complement.  Every point \(x\in\R^d\) can be written uniquely as \(x=z+u\), where \(z\in E^\perp\) and \(u\in E\).  Lebesgue measure on \(E\) is denoted by \(\dd u\).

Define the directional smoothness of \(f\) along \(E\) by
\begin{equation}
\label{eq:reftrace-directional-L}
L_E := \sup_{x\in\R^d}\|P_E\nabla^2 f(x)P_E\|_{\op}.
\end{equation}
By \Cref{ass:f}, \(L_E\le L_f\).  Define the directional Lipschitz size of \(g\) along \(E\) by
\begin{equation}
\label{eq:reftrace-directional-G}
G_E := \sup\bigl\{\|P_E s\|:x\in\R^d,\ s\in\partial g(x)\bigr\}.
\end{equation}
Since \(g\) is \(G\)-Lipschitz on \(\R^d\), \(G_E\le G\).  Moreover, by the proximal optimality condition,
\begin{equation}
\label{eq:reftrace-moreau-subgradient}
\nabla g_\lambda(x)\in \partial g(p_\lambda(x)),
\qquad p_\lambda(x)=\prox_{\lambda g}(x),
\end{equation}
so \(\|P_E\nabla g_\lambda(x)\|\le G_E\) for every \(x\).

Finally set
\begin{equation}
\label{eq:reftrace-BE}
B_E :=
\left[
\int_E
\exp\left\{-\frac{L_E}{2}\|u\|^2-2G_E\|u\|\right\}\,\dd u
\right]^{-1}.
\end{equation}
This constant is the scale of the slice-density bound below. It depends
on the chosen subspace \(E\) only through \(q\), \(L_E\), and \(G_E\), and
it does not involve the Moreau parameter \(\lambda\).

When \(E=\mathrm{span}(v)\) for a unit vector \(v\), we write \(L_v\), \(G_v\), and \(B_v\) in place of \(L_E\), \(G_E\), and \(B_E\). In this case
\begin{equation}
\label{eq:reftrace-Bv}
B_v=
\left[
\int_{\R}
\exp\left\{-\frac{L_v}{2}s^2-2G_v|s|\right\}\,\dd s
\right]^{-1}.
\end{equation}

\begin{lemma}[Slice conditional density]
\label{lem:reftrace-slice-density}
Let \(X^{\rm ref}\sim\pi_\lambda\).  For every \(z\in E^\perp\), the conditional density of \(P_E X^{\rm ref}\) on the slice \(z+E\) is bounded by \(B_E\).  Equivalently, the density proportional to
\begin{equation}
 u\mapsto \exp\{-U_\lambda(z+u)\},
 \qquad u\in E,
\end{equation}
has \(L^\infty(E)\)-norm at most \(B_E\).
\end{lemma}

\begin{proof}
Fix \(z\in E^\perp\) and write
\begin{equation}
\phi_z(u):=U_\lambda(z+u),
\qquad u\in E.
\end{equation}
The function \(\phi_z\) is convex and coercive on \(E\), because \(f\) is strongly convex and \(g_\lambda\) is convex.  Let \(u_z\) be a minimizer of \(\phi_z\).  The first-order condition on the slice gives
\begin{equation}
P_E\nabla f(z+u_z)+P_E\nabla g_\lambda(z+u_z)=0.
\end{equation}
By \eqref{eq:reftrace-moreau-subgradient},
\begin{equation}
\|P_E\nabla f(z+u_z)\|
=\|P_E\nabla g_\lambda(z+u_z)\|
\le G_E.
\end{equation}
For any \(v\in E\), the definition of \(L_E\) gives
\begin{equation}
f(z+u_z+v)
\le f(z+u_z)+\langle P_E\nabla f(z+u_z),v\rangle
+\frac{L_E}{2}\|v\|^2.
\end{equation}
Also, because \(g_\lambda\) is \(G_E\)-Lipschitz along \(E\),
\begin{equation}
g_\lambda(z+u_z+v)
\le g_\lambda(z+u_z)+G_E\|v\|.
\end{equation}
Combining the previous three displays,
\begin{equation}
\phi_z(u_z+v)
\le \phi_z(u_z)+\frac{L_E}{2}\|v\|^2+2G_E\|v\|.
\end{equation}
The conditional normalizing constant on the slice is therefore at least
\begin{equation}
e^{-\phi_z(u_z)}
\int_E
\exp\left\{-\frac{L_E}{2}\|v\|^2-2G_E\|v\|\right\}\,\dd v.
\end{equation}
Since the conditional density is maximized at a minimizer of \(\phi_z\), its maximum is bounded by \(B_E\).
\end{proof}

\subsection{Density propagation by \(T_h\) and by the heat step}

The previous lemma gives slice-density bounds for the stationary
input \(X^{\rm ref}\sim\pi_\lambda\). The reference active trace,
however, is evaluated along the one-step reference heat path
\[
Y_t^{\rm ref}=T_h(X^{\rm ref})+\sqrt{2t}Z .
\]
We therefore need to propagate the slice-density bound through the
deterministic Euler map \(T_h\) and then through the Gaussian heat step.

The deterministic part requires a small-step condition ensuring that
\(T_h\) is injective on each \(E\)-slice. Once this is available, the heat
step is harmless: convolution with a probability density cannot increase
an \(L^\infty\) density bound.

For the same subspace \(E\), define
\begin{equation}
\label{eq:reftrace-alpha-E}
\alpha_E := 1-h(L_E+\lambda^{-1}).
\end{equation}
The useful case is \(\alpha_E>0\).

\begin{lemma}[Slice density propagation]
\label{lem:reftrace-density-propagation}
Let \(X^{\rm ref}\sim\pi_\lambda\) and \(Y_t^{\rm ref}=T_h(X^{\rm ref})+\sqrt{2t}\,Z\), where \(Z\sim N(0,I_d)\) is independent of \(X^{\rm ref}\).  If \(\alpha_E>0\), then, for every Borel set \(A\subset E\) and every \(0\le t\le h\),
\begin{equation}
\label{eq:reftrace-density-propagation}
\Pp\{P_E Y_t^{\rm ref}\in A\}
\le \frac{B_E}{\alpha_E^q}\,\vol_E(A),
\end{equation}
where \(\vol_E\) denotes Lebesgue measure on \(E\).
\end{lemma}

\begin{proof}
For fixed \(z\in E^\perp\), consider the map from the slice \(z+E\) to the \(E\)-coordinate after the deterministic step:
\begin{equation}\label{eq:reftrace-Sz}
S_z(u):=P_E T_h(z+u),
\qquad u\in E.
\end{equation}
We first study the deterministic map \(S_z\) on each slice. Fix \(z\in E^\perp\). For \(u,w\in E\), the fundamental theorem of
calculus and the definition of \(L_E\) give
\begin{align*}
    &\left\|
        P_E\bigl\{
            \nabla f(z+u)-\nabla f(z+w)
        \bigr\}
    \right\|=
    \left\|
        \int_0^1
        P_E\nabla^2 f\bigl(z+w+s(u-w)\bigr)P_E(u-w)
        \,ds
    \right\|\le L_E\|u-w\|.
\end{align*}
Here we used \(P_E(u-w)=u-w\). Moreover, by
\Cref{lem:moreau-regularity},
\(\operatorname{Lip}(\nabla g_\lambda)\le\lambda^{-1}\), and hence
\[
    \left\|
        P_E\bigl\{
            \nabla g_\lambda(z+u)-\nabla g_\lambda(z+w)
        \bigr\}
    \right\|
    \le
    \lambda^{-1}\|u-w\|.
\]
Since \(U_\lambda=f+g_\lambda\), it follows that
\begin{equation}
    \left\|
        P_E\bigl\{
            \nabla U_\lambda(z+u)-\nabla U_\lambda(z+w)
        \bigr\}
    \right\|
    \le
    (L_E+\lambda^{-1})\|u-w\|.
    \label{eq:slice-gradient-lipschitz}
\end{equation}

Using the definition of \(S_z\) in \eqref{eq:reftrace-Sz}, the reverse
triangle inequality and \eqref{eq:slice-gradient-lipschitz} yield
\begin{align}
    \|S_z(u)-S_z(w)\|
    &=
    \left\|
        u-w
        -hP_E\bigl\{
            \nabla U_\lambda(z+u)-\nabla U_\lambda(z+w)
        \bigr\}
    \right\|
    \notag\\
    &\ge
    \|u-w\|
    -
    h\left\|
        P_E\bigl\{
            \nabla U_\lambda(z+u)-\nabla U_\lambda(z+w)
        \bigr\}
    \right\|
    \notag\\
    &\ge
    \bigl\{1-h(L_E+\lambda^{-1})\bigr\}\|u-w\|
    \notag\\
    &=
    \alpha_E\|u-w\|.
    \label{eq:slice-lower-lipschitz}
\end{align}
In particular, \(S_z\) is injective.

We next prove that \(S_z\) is surjective. Fix \(y\in E\). Consider the
self-map of \(E\) given by
\[
    u\longmapsto y+hP_E\nabla U_\lambda(z+u).
\]
By \eqref{eq:slice-gradient-lipschitz}, this map is Lipschitz with
constant
\[
    h(L_E+\lambda^{-1})=1-\alpha_E<1.
\]
Since \(E\) is a finite-dimensional Hilbert space and therefore
complete, the Banach fixed-point theorem gives a unique \(u\in E\)
such that
\[
    u=y+hP_E\nabla U_\lambda(z+u).
\]
By \eqref{eq:reftrace-Sz}, this identity is equivalent to
\(S_z(u)=y\). Since \(y\in E\) was arbitrary, \(S_z:E\to E\) is
surjective and hence bijective.

Furthermore, \eqref{eq:slice-lower-lipschitz} implies that its inverse
\(S_z^{-1}:E\to E\) satisfies
\[
    \|S_z^{-1}(y)-S_z^{-1}(y')\|
    \le
    \alpha_E^{-1}\|y-y'\|,
    \qquad y,y'\in E.
\]
Thus \(S_z^{-1}\) is globally \(\alpha_E^{-1}\)-Lipschitz on all of
\(E\).

After choosing an orthonormal basis of \(E\), we identify \(E\)
isometrically with \(\mathbb{R}^q\). The standard measure-distortion
inequality for Lipschitz maps
\cite[Theorem~2.8]{evans2025measure} can therefore be applied directly
to the globally defined map \(S_z^{-1}:E\to E\). Consequently, for
every Borel set \(A\subset E\),
\begin{equation}
    \operatorname{vol}_E\bigl(S_z^{-1}(A)\bigr)
    \le
    \operatorname{Lip}(S_z^{-1})^q
    \operatorname{vol}_E(A)
    \le
    \alpha_E^{-q}\operatorname{vol}_E(A).
    \label{eq:slice-preimage-volume}
\end{equation}
Here
\[
    S_z^{-1}(A)=\{u\in E:S_z(u)\in A\}.
\]
This set is Borel because \(S_z\) is continuous.

We now propagate the conditional slice-density bound. For
\(z\in E^\perp\), define
\[
    \rho_{\lambda,z}(u)
    :=
    \frac{\exp\{-U_\lambda(z+u)\}}
    {\displaystyle
        \int_E\exp\{-U_\lambda(z+v)\}\,\dd v},
    \qquad u\in E.
\]
The kernel \(\rho_{\lambda,z}(u)\,\dd u\) is a version of the conditional
law of \(P_EX^{\rm ref}\) given
\(P_{E^\perp}X^{\rm ref}=z\). By
\Cref{lem:reftrace-slice-density},
\[
    \|\rho_{\lambda,z}\|_{L^\infty(E)}
    \le B_E.
\]
Therefore, for every Borel set \(A\subset E\),
\begin{align*}
    &\mathbb{P}\bigl\{
        P_ET_h(X^{\rm ref})\in A
        \,\bigm|\,
        P_{E^\perp}X^{\rm ref}=z
    \bigr\}
    \\
    &\quad=
    \int_{S_z^{-1}(A)}
        \rho_{\lambda,z}(u)\,\dd u
    \\
    &\quad\le
    B_E\operatorname{vol}_E\bigl(S_z^{-1}(A)\bigr)
    \\
    &\quad\le
    B_E\alpha_E^{-q}\operatorname{vol}_E(A),
\end{align*}
where the last inequality follows from
\eqref{eq:slice-preimage-volume}. The bound is uniform in
\(z\in E^\perp\). Integrating with respect to the law of
\(P_{E^\perp}X^{\rm ref}\) gives
\begin{equation}
    \mathbb{P}\{P_ET_h(X^{\rm ref})\in A\}
    \le
    B_E\alpha_E^{-q}\operatorname{vol}_E(A).
    \label{eq:deterministic-slice-density}
\end{equation}

Finally, let \(0\le t\le h\). Since \(Z\) is independent of
\(X^{\rm ref}\), \eqref{eq:deterministic-slice-density} applied
conditionally on \(Z\) gives
\begin{align*}
    \mathbb{P}\{P_EY_t^{\rm ref}\in A\}
    &=
    \mathbb{E}\left[
        \mathbb{P}\left\{
            P_ET_h(X^{\rm ref})
            \in
            A-\sqrt{2t}\,P_EZ
            \,\middle|\,
            Z
        \right\}
    \right]
    \\
    &\le
    B_E\alpha_E^{-q}
    \mathbb{E}\left[
        \operatorname{vol}_E
        \bigl(A-\sqrt{2t}\,P_EZ\bigr)
    \right]
    \\
    &=
    B_E\alpha_E^{-q}\operatorname{vol}_E(A),
\end{align*}
where the last equality follows from translation invariance of
Lebesgue measure on \(E\). This proves
\eqref{eq:reftrace-density-propagation} simultaneously for all
\(0\le t\le h\).
\end{proof}

\subsection{Slab, ball, and inverse-radius consequences}

\Cref{lem:reftrace-density-propagation} immediately gives the estimates used in the examples.  We state them separately to make later verification modules short.

\begin{corollary}[Slab probability]
\label{cor:reftrace-slab}
Let \(v\in\R^d\) be a unit vector and let \(\alpha_v=1-h(L_v+\lambda^{-1})>0\).  Then, for every \(a\in\R\), \(r>0\), and \(0\le t\le h\),
\begin{equation}
\label{eq:reftrace-slab-prob}
\Pp\{|\langle v,Y_t^{\rm ref}\rangle-a|\le r\}
\le \frac{2rB_v}{\alpha_v}.
\end{equation}
Equivalently, if \(d_0\in\R^d\setminus\{0\}\) and \(v=d_0/\|d_0\|\), then
\begin{equation}
\label{eq:reftrace-slab-prob-linear-form}
\Pp\{|d_0^\top Y_t^{\rm ref}-a|\le r\}
\le \frac{2rB_v}{\|d_0\|\alpha_v}.
\end{equation}
\end{corollary}

\begin{proof}
Apply \Cref{lem:reftrace-density-propagation} with \(E=\mathrm{span}(v)\) and with \(A=\{sv: s\in[a-r,a+r]\}\).  For \eqref{eq:reftrace-slab-prob-linear-form}, rewrite the event as
\begin{equation}
\left|\left\langle v,Y_t^{\rm ref}\right\rangle-\frac{a}{\|d_0\|}\right|
\le \frac{r}{\|d_0\|}.
\end{equation}
\end{proof}

\begin{corollary}[Ball probability]
\label{cor:reftrace-ball}
Let \(E\subset\R^d\) have dimension \(q\), and assume \(\alpha_E>0\).  Let \(v_q\) be the volume of the unit ball in \(\R^q\).  Then, for every \(y\in E\), \(r>0\), and \(0\le t\le h\),
\begin{equation}
\label{eq:reftrace-ball-prob}
\Pp\{\|P_E Y_t^{\rm ref}-y\|\le r\}
\le \frac{B_E}{\alpha_E^q}v_qr^q.
\end{equation}
\end{corollary}

\begin{proof}
Apply \Cref{lem:reftrace-density-propagation} to the ball \(A=\{u\in E:\|u-y\|\le r\}\), whose \(E\)-volume is \(v_qr^q\).
\end{proof}

\begin{corollary}[Inverse-radius bound]
\label{cor:reftrace-inv-radius}
Let \(E\subset\R^d\) have dimension \(q\ge 2\), and assume \(\alpha_E>0\).  Then, for every \(y\in E\) and \(0\le t\le h\),
\begin{equation}
\label{eq:reftrace-inv-radius-bound}
\E\frac{1}{\|P_E Y_t^{\rm ref}-y\|}
\le C_q\left(\frac{B_E}{\alpha_E^q}\right)^{1/q},
\qquad
C_q:=\frac{q}{q-1}v_q^{1/q}.
\end{equation}
\end{corollary}

\begin{proof}
Let \(W=P_E Y_t^{\rm ref}\).  By \Cref{lem:reftrace-density-propagation}, \(W\) has density at most
\begin{equation}
\overline B_E:=\frac{B_E}{\alpha_E^q}.
\end{equation}
Therefore, for every \(s>0\),
\begin{equation}
\Pp\{\|W-y\|^{-1}>s\}
=\Pp\{\|W-y\|<s^{-1}\}
\le \min\{1,\overline B_Ev_qs^{-q}\}.
\end{equation}
Integrating the tail bound gives
\begin{align}
\E\frac{1}{\|W-y\|}
&=\int_0^\infty \Pp\{\|W-y\|^{-1}>s\}\,\dd s 
\notag\\
&\le \int_0^{(\overline B_Ev_q)^{1/q}}1\,\dd s
+\overline B_Ev_q\int_{(\overline B_Ev_q)^{1/q}}^\infty s^{-q}\,\dd s
\notag\\
&=\frac{q}{q-1}(\overline B_Ev_q)^{1/q}.
\end{align}
\end{proof}

\section{Examples and Verification Modules}\label{sec:examples-verification}

This section applies the reference active-trace framework to four classes of
nonsmooth penalties.  We retain the notation of the previous sections.  Let
\(X^{\rm ref}\sim\pi_\lambda\), let
\((W_t^{\rm ref})_{t\ge0}\) be a standard Brownian motion in
\(\mathbb R^d\), independent of \(X^{\rm ref}\), and define
\[
    Y_t^{\rm ref}
    :=
    T_h(X^{\rm ref})+\sqrt{2}\,W_t^{\rm ref},
    \qquad 0\le t\le h,
\]
where $T_h(x)=x-h\nabla U_\lambda(x), U_\lambda=f+g_\lambda$.
The corresponding reference active trace is
\[
    B_{\rm ref}
    :=
    \frac1h\int_0^h
    \mathbb E\,a_\lambda(Y_t^{\rm ref})\,\dd t.
\]

Each example has the same structure: we first identify the region on which
the Moreau curvature \(a_\lambda\) is active, then control the mass of that
region under the reference heat path using the slice-density estimates of
\Cref{sec:reftrace}, and finally substitute the resulting bound on \(B_{\rm ref}\)
into \Cref{thm:main-complete}.

We use the symmetric
error split \(\varepsilon_{\rm alg}=\varepsilon_{\rm bias}=\varepsilon/2\) and the universal
Moreau-bias choice
\begin{equation}\label{eq:examples-lambda-choice}
        \lambda=\frac{2\varepsilon}{G^2}.
\end{equation}

\paragraph{Automatic compatibility with the density-propagation estimates.}
For a subspace \(E\), \Cref{lem:reftrace-density-propagation} is applied with
\[
    \alpha_E
    :=
    1-h(L_E+\lambda^{-1}).
\]
In the estimates below, it is sufficient to have
\(\alpha_E\ge1/2\).  This is not an additional assumption. It is automatic under the
step-size requirement \eqref{eq:main-h-choice} in \Cref{thm:main-complete}. 

Indeed, the \(L_f^{-1}\) restriction gives
\[
    hL_f\le c,
\]
whereas the
\(\varepsilon_{\rm alg}^2/(\tau_f+G^2+B_{\rm ref})\) restriction gives,
after adjusting the universal constant \(c\),
\[
    h\le c\frac{\varepsilon^2}{G^2}.
\]
Since \(L_E\le L_f\), \(\lambda=2\varepsilon/G^2\), and
\(0<\varepsilon\le1\), it follows that
\[
\begin{aligned}
    h(L_E+\lambda^{-1})
    \le hL_f+\frac{hG^2}{2\varepsilon} 
    \le c+\frac{c\varepsilon}{2}
    \le \frac{3c}{2}.
\end{aligned}
\]
Taking the universal constant \(c\) sufficiently small, every step-size
choice below therefore satisfies, for each relevant subspace \(E\),
\begin{equation}\label{eq:examples-density-compatibility}
    h(L_E+\lambda^{-1})\le\frac12,
    \qquad
    \alpha_E\ge\frac12.
\end{equation}

\subsection{One-dimensional finite-kink piecewise-linear penalties}\label{subsec:oned-pl-example}

Let \(d=1\), and let \(g:\R\to\R\) be a convex piecewise-linear function with finitely many kink
points \(\theta_j\).  Denote the left and right slopes at \(\theta_j\) by \(s_j^-\) and \(s_j^+\), and
put
\begin{equation}\label{eq:oned-pl-jumps}
        \Delta_j=s_j^+-s_j^- >0,
        \qquad
        \Delta=\sum_j \Delta_j .
\end{equation}
Let \(G=\Lip(g)=\sup |s|\), where the supremum is over the slopes of \(g\), and define
\begin{equation}\label{eq:oned-pl-density-constant}
        B=\left[\int_{\R}\exp\left\{-\frac{L_f}{2}u^2-2G|u|\right\}\,\dd u\right]^{-1}.
\end{equation}

\begin{lemma}[Moreau curvature of a one-dimensional finite-kink PL function]
\label{lem:oned-pl-moreau-curvature}
There is a Borel set \(A_\lambda\subset\R\) such that
\begin{equation}\label{eq:oned-pl-curvature-set}
        g_\lambda''(x)\le \frac1\lambda \mathbf 1_{A_\lambda}(x)
        \quad\text{for a.e. }x,
        \qquad
        |A_\lambda|\le \lambda\Delta .
\end{equation}
Indeed, we may take
\begin{equation}\label{eq:oned-pl-active-intervals}
        A_\lambda=\bigcup_j [\theta_j+\lambda s_j^-,\,\theta_j+\lambda s_j^+].
\end{equation}
\end{lemma}

\begin{proof}
Let \(p_\lambda=\prox_{\lambda g}\).  The optimality condition is
\begin{equation}\label{eq:oned-pl-kkt}
        \frac{x-p_\lambda(x)}{\lambda}\in \partial g(p_\lambda(x)).
\end{equation}
Consider an open affine piece \(I\) of \(g\), and let its slope be \(s_I\).  If
\(p_\lambda(x_0)\in I\), then \(x_0=p_\lambda(x_0)+\lambda s_I\).  For all \(x\) close enough to
\(x_0\), the point \(x-\lambda s_I\) still lies in \(I\).  It satisfies
\((x-(x-\lambda s_I))/\lambda=s_I\in\partial g(x-\lambda s_I)\), so by uniqueness of the proximal
point,
\[
        p_\lambda(x)=x-\lambda s_I
\]
for all such \(x\).  Hence \(Dp_\lambda(x)=1\) locally on this set.  By the Moreau identity
\(g_\lambda''=\lambda^{-1}(1-Dp_\lambda)\) at a.e. differentiability points of \(p_\lambda\), the
Moreau curvature is zero whenever the proximal point belongs to the interior of an affine piece.

Therefore nonzero curvature can occur only at points \(x\) for which \(p_\lambda(x)\) is a kink.  If
\(p_\lambda(x)=\theta_j\), then \eqref{eq:oned-pl-kkt} gives
\[
        x=\theta_j+\lambda s,
        \qquad s\in\partial g(\theta_j)=[s_j^-,s_j^+].
\]
Thus such \(x\)'s lie in the interval
\([\theta_j+\lambda s_j^-,\theta_j+\lambda s_j^+]\), whose length is \(\lambda\Delta_j\).  Taking
the union over all kinks gives \eqref{eq:oned-pl-active-intervals} and
\(|A_\lambda|\le\sum_j\lambda\Delta_j=\lambda\Delta\).  The pointwise bound
\(0\le g_\lambda''\le\lambda^{-1}\) gives \eqref{eq:oned-pl-curvature-set}.
\end{proof}

\begin{proposition}[One-dimensional finite-kink PL reference trace]
\label{prop:oned-pl-reference-trace}
Let
\begin{equation}\label{eq:oned-pl-A1d}
        A_{\rm 1d}:=2B\Delta .
\end{equation}
There exist universal constants \(c,C>0\) such that the following holds.  For
\(0<\varepsilon\le1\), assume \(G>0\), set
\begin{equation}\label{eq:oned-pl-lambda-choice}
        \lambda=\frac{2\varepsilon}{G^2},
\end{equation}
and choose the step size
\begin{equation}\label{eq:oned-pl-step-choice}
        h
        =c\min\left\{
              L_f^{-1},
              \frac{\varepsilon^2}{\tau_f+G^2+A_{\rm 1d}},
              \varepsilon\lambda
        \right\}.
\end{equation}
Then
\begin{equation}\label{eq:oned-pl-bref}
        B_{\rm ref}\le A_{\rm 1d},
        \qquad
        M_\lambda=\frac1\lambda .
\end{equation}
If
\begin{equation}\label{eq:oned-pl-N-choice}
        N\ge \frac{C}{mh}
        \log\left(1+\frac{m\Phi_0(\lambda,h)}{\varepsilon^2}\right),
\end{equation}
then \(\sqrt m W_2(\mu_N,\pi)\le\varepsilon\).  Equivalently, with the step size
\eqref{eq:oned-pl-step-choice}, it is sufficient to take
\begin{equation}\label{eq:oned-pl-complexity}
        N_{\rm 1d}(\varepsilon)
        \le
        \frac{C}{m}
        \left[
              L_f+
              \frac{\tau_f+G^2+A_{\rm 1d}}{\varepsilon^2}
        \right]
        \log\left(1+\frac{m\Phi_0(\lambda,h)}{\varepsilon^2}\right).
\end{equation}
\end{proposition}

\begin{proof}
By \Cref{lem:oned-pl-moreau-curvature},
\[
        B_{\rm ref}
        \le
        \frac1{\lambda h}\int_0^h \Pp\{Y_t^{\rm ref}\in A_\lambda\}\,\dd t.
\]
The step size \eqref{eq:oned-pl-step-choice} implies \eqref{eq:examples-density-compatibility}, as explained
at the beginning of the section.  By \Cref{lem:reftrace-density-propagation}, 
\[
        \Pp\{Y_t^{\rm ref}\in A_\lambda\}
        \le 2B|A_\lambda|
        \le 2B\lambda\Delta .
\]
This proves \(B_{\rm ref}\le A_{\rm 1d}\).  The value of \(M_\lambda\) is the one-dimensional
global trace bound.  The choices \eqref{eq:oned-pl-lambda-choice} and
\eqref{eq:oned-pl-step-choice} are exactly \Cref{thm:main-complete}, with the symmetric error split and with
\(B_{\rm ref}\le A_{\rm 1d}\), up to universal constants.  This gives
\(\sqrt mW_2(\mu_N,\pi)\le\varepsilon\) under \eqref{eq:oned-pl-N-choice}.  Taking the reciprocal
of \eqref{eq:oned-pl-step-choice} and using \(1/(\varepsilon\lambda)=G^2/(2\varepsilon^2)\) gives
\eqref{eq:oned-pl-complexity}.
\end{proof}

\subsection{Separable finite-kink PL penalties and weighted lasso}\label{subsec:separable-pl-example}

Now let \(d\ge1\) and consider
\begin{equation}\label{eq:separable-pl-penalty}
        g(x)=\sum_{i=1}^d g_i(x_i),
\end{equation}
where each \(g_i:\R\to\R\) is convex, Lipschitz, and piecewise-linear with finitely many kinks.  Let
\(G_i=\Lip(g_i)\), and let \(\Delta_i\) be the total slope jump of \(g_i\), as in
\eqref{eq:oned-pl-jumps}.  A global Lipschitz constant of \(g\) is
\begin{equation}\label{eq:separable-pl-global-G}
        G=\left(\sum_{i=1}^d G_i^2\right)^{1/2}.
\end{equation}
For each coordinate set
\begin{equation}\label{eq:separable-pl-directional-constants}
        L_i=\sup_{x\in\R^d}\partial_{ii}^2 f(x),
        \qquad
        B_i=\left[\int_{\R}\exp\left\{-\frac{L_i}{2}u^2-2G_i|u|\right\}\,\dd u\right]^{-1}.
\end{equation}
The smooth part \(f\) is not assumed to be separable.  The constants \(L_i\) and \(B_i\) only
control one-dimensional conditional slices in coordinate directions.

\begin{proposition}[Separable finite-kink PL reference trace]
\label{prop:separable-pl-reference-trace}
Let
\begin{equation}\label{eq:separable-pl-Asep}
        A_{\rm sep}:=2\sum_{i=1}^d B_i\Delta_i .
\end{equation}
There exist universal constants \(c,C>0\) such that the following holds.  For
\(0<\varepsilon\le1\), assume \(G>0\), set
\begin{equation}\label{eq:separable-pl-lambda-choice}
        \lambda=\frac{2\varepsilon}{G^2},
\end{equation}
and choose
\begin{equation}\label{eq:separable-pl-step-choice}
        h
        =c\min\left\{
              L_f^{-1},
              \frac{\varepsilon^2}{\tau_f+G^2+A_{\rm sep}},
              \frac{\varepsilon\lambda}{d}
        \right\}.
\end{equation}
Then
\begin{equation}\label{eq:separable-pl-bref}
        B_{\rm ref}\le A_{\rm sep},
        \qquad
        M_\lambda= \frac d\lambda .
\end{equation}
If
\begin{equation}\label{eq:separable-pl-N-choice}
        N\ge \frac{C}{mh}
        \log\left(1+\frac{m\Phi_0(\lambda,h)}{\varepsilon^2}\right),
\end{equation}
then \(\sqrt m W_2(\mu_N,\pi)\le\varepsilon\).  Equivalently, with the step size
\eqref{eq:separable-pl-step-choice}, it is sufficient to take
\begin{equation}\label{eq:separable-pl-complexity}
        N_{\rm sep}(\varepsilon)
        \le
        \frac{C}{m}
        \left[
              L_f+
              \frac{\tau_f+A_{\rm sep}+dG^2}{\varepsilon^2}
        \right]
        \log\left(1+\frac{m\Phi_0(\lambda,h)}{\varepsilon^2}\right).
\end{equation}
\end{proposition}

\begin{proof}
The Moreau envelope of a sum of coordinate functions is the sum of their one-dimensional Moreau
envelopes.  \Cref{lem:oned-pl-moreau-curvature} gives sets
\(A_{i,\lambda}\subset\R\) with \(|A_{i,\lambda}|\le\lambda\Delta_i\) such that
\begin{equation}\label{eq:separable-pl-curvature-decomposition}
        a_\lambda(x)
        \le
        \frac1\lambda\sum_{i=1}^d \mathbf 1_{\{x_i\in A_{i,\lambda}\}}
        \quad\text{for a.e. }x.
\end{equation}
The step size \eqref{eq:separable-pl-step-choice} implies \eqref{eq:examples-density-compatibility}.  By \Cref{lem:reftrace-density-propagation} in the \(i\)-th coordinate direction, 
\[
        \Pp\{Y_{t,i}^{\rm ref}\in A_{i,\lambda}\}
        \le 2B_i|A_{i,\lambda}|
        \le 2B_i\lambda\Delta_i .
\]
Substitution in \eqref{eq:separable-pl-curvature-decomposition} and time averaging prove
\(B_{\rm ref}\le A_{\rm sep}\).  The trace bound \(M_\lambda= d/\lambda\) follows from
\(0\preceq H_\lambda\preceq \lambda^{-1}\Id\).  The choices
\eqref{eq:separable-pl-lambda-choice} and \eqref{eq:separable-pl-step-choice} are \Cref{thm:main-complete} with
\(B_{\rm ref}\le A_{\rm sep}\) and \(M_\lambda= d/\lambda\), up to universal constants.  Since
\(d/(\varepsilon\lambda)=dG^2/(2\varepsilon^2)\), the displayed complexity follows.
\end{proof}

\begin{corollary}[Weighted lasso]\label{cor:weighted-lasso-reference-trace}
Let
\begin{equation}\label{eq:weighted-lasso-penalty}
        g(x)=\sum_{i=1}^d \gamma_i |x_i|,
        \qquad \gamma_i\ge0.
\end{equation}
Then \(G^2=\sum_i\gamma_i^2\), \(\Delta_i=2\gamma_i\), and
\begin{equation}\label{eq:weighted-lasso-Awl}
        A_{\rm wl}:=4\sum_{i=1}^d \gamma_iB_i .
\end{equation}
There exist universal constants \(c,C>0\) such that, for \(0<\varepsilon\le1\) and \(G>0\), set
\begin{equation}\label{eq:weighted-lasso-lambda-choice}
        \lambda=\frac{2\varepsilon}{G^2},
\end{equation}
and choose
\begin{equation}\label{eq:weighted-lasso-step-choice}
        h
        =c\min\left\{
              L_f^{-1},
              \frac{\varepsilon^2}{\tau_f+G^2+A_{\rm wl}},
              \frac{\varepsilon\lambda}{d}
        \right\}.
\end{equation}
Then
\begin{equation}\label{eq:weighted-lasso-bref}
        B_{\rm ref}\le A_{\rm wl},
        \qquad
        M_\lambda= \frac d\lambda .
\end{equation}
If
\begin{equation}\label{eq:weighted-lasso-N-choice}
        N\ge \frac{C}{mh}
        \log\left(1+\frac{m\Phi_0(\lambda,h)}{\varepsilon^2}\right),
\end{equation}
then \(\sqrt m W_2(\mu_N,\pi)\le\varepsilon\).  In particular,
\begin{equation}\label{eq:weighted-lasso-complexity}
        N_{\rm wl}(\varepsilon)
        \le
        \frac{C}{m}
        \left[
              L_f+
              \frac{\tau_f+A_{\rm wl}+dG^2}{\varepsilon^2}
        \right]
        \log\left(1+\frac{m\Phi_0(\lambda,h)}{\varepsilon^2}\right).
\end{equation}
\end{corollary}

\subsection{Group lasso}\label{subsec:group-lasso-example}

Let \(\{1,\ldots,d\}\) be partitioned into disjoint blocks \(b\).  The block variable is denoted by
\(x_b\in\R^{q_b}\), where \(q_b\) is the block size and \(\sum_b q_b=d\).  Consider
\begin{equation}\label{eq:group-lasso-penalty}
        g(x)=\sum_b \gamma_b\|x_b\|,
        \qquad \gamma_b>0.
\end{equation}
A global Lipschitz constant is
\begin{equation}\label{eq:group-lasso-global-G}
        G=\left(\sum_b \gamma_b^2\right)^{1/2}.
\end{equation}
Let \(P_b\) be the orthogonal projection onto block \(b\).  Define
\begin{equation}\label{eq:group-lasso-density-constants}
        L_b=\sup_{x\in\R^d}\|P_b\nabla^2 f(x)P_b\|_{\rm op},
        \qquad
        B_b=\left[\int_{\R^{q_b}}
             \exp\left\{-\frac{L_b}{2}\|u\|^2-2\gamma_b\|u\|\right\}\,\dd u\right]^{-1}.
\end{equation}
Let \(v_q\) be the volume of the unit ball in \(\R^q\).  For \(q\ge2\), set
\begin{equation}\label{eq:inverse-radius-constant}
        C_q=\frac{q}{q-1}v_q^{1/q}.
\end{equation}

\begin{lemma}[Block Moreau trace for \(\gamma\|z\|\)]\label{lem:group-lasso-block-trace}
Let \(q\ge1\), \(\gamma>0\), and \(\varphi(z)=\gamma\|z\|\) on \(\R^q\).  For
\(r=\|z\|\),
\begin{equation}\label{eq:group-lasso-block-trace}
        \tr\nabla^2\varphi_\lambda(z)
        =
        \frac q\lambda\mathbf 1_{\{r\le \lambda\gamma\}}
        +
        \frac{\gamma(q-1)}{r}\mathbf 1_{\{r>\lambda\gamma\}}
        \quad\text{for a.e. }z.
\end{equation}
\end{lemma}

\begin{proof}
The proximal map of \(\lambda\gamma\|\cdot\|\) is block soft-thresholding:
\[
        \prox_{\lambda\varphi}(z)=\left(1-\frac{\lambda\gamma}{\|z\|}\right)_+z.
\]
If \(r<\lambda\gamma\), the proximal map is locally constant, so
\(\nabla^2\varphi_\lambda=\lambda^{-1}I_q\) and the trace is \(q/\lambda\).  If
\(r>\lambda\gamma\), then \(\varphi_\lambda(z)=\gamma r-\lambda\gamma^2/2\).  The Hessian of
\(r\) has eigenvalue \(0\) in the radial direction and eigenvalue \(1/r\) in the \(q-1\) tangential
directions.  The sphere \(r=\lambda\gamma\) is null and its value is irrelevant.
\end{proof}

\begin{proposition}[Group lasso reference trace]\label{prop:group-lasso-reference-trace}
Let
\begin{equation}\label{eq:group-lasso-Agrp-definition}
        A_{\rm grp}
        :=
        \sum_b
             q_b v_{q_b}2^{q_b}B_b\gamma_b^{q_b}
             +
        \sum_{b:q_b\ge2}
             2\gamma_b(q_b-1)C_{q_b}B_b^{1/q_b}.
\end{equation}
There exist universal constants \(c,C>0\) such that the following holds.  For
\(0<\varepsilon\le \min\{1,G^2/2\}\), assume \(G>0\), set
\begin{equation}\label{eq:group-lasso-lambda-choice}
        \lambda=\frac{2\varepsilon}{G^2},
\end{equation}
and choose
\begin{equation}\label{eq:group-lasso-step-choice}
        h
        =c\min\left\{
              L_f^{-1},
              \frac{\varepsilon^2}{\tau_f+G^2+A_{\rm grp}},
              \frac{\varepsilon\lambda}{d}
        \right\}.
\end{equation}
Then \(\lambda\le1\) and
\begin{equation}\label{eq:group-lasso-Agrp}
        B_{\rm ref}\le A_{\rm grp},
        \qquad
        M_\lambda= \frac d\lambda .
\end{equation}
If
\begin{equation}\label{eq:group-lasso-N-choice}
        N\ge \frac{C}{mh}
        \log\left(1+\frac{m\Phi_0(\lambda,h)}{\varepsilon^2}\right),
\end{equation}
then \(\sqrt m W_2(\mu_N,\pi)\le\varepsilon\).  Equivalently, with the step size
\eqref{eq:group-lasso-step-choice}, it is sufficient to take
\begin{equation}\label{eq:group-lasso-complexity}
        N_{\rm grp}(\varepsilon)
        \le
        \frac{C}{m}
        \left[
              L_f+
              \frac{\tau_f+A_{\rm grp}+dG^2}{\varepsilon^2}
        \right]
        \log\left(1+\frac{m\Phi_0(\lambda,h)}{\varepsilon^2}\right).
\end{equation}
\end{proposition}

\begin{proof}
The penalty is separable across blocks, so the Moreau envelope and the active trace decompose
across blocks.  The step size \eqref{eq:group-lasso-step-choice} implies
\eqref{eq:examples-density-compatibility}.  Therefore \Cref{lem:reftrace-density-propagation} gives a density bound \(2^{q_b}B_b\) for
\(P_b Y_t^{\rm ref}\), uniformly in \(t\).

For the center term in \Cref{lem:group-lasso-block-trace},
\[
        \frac{q_b}{\lambda}
        \Pp\{\|P_b Y_t^{\rm ref}\|\le \lambda\gamma_b\}
        \le
        \frac{q_b}{\lambda}
        2^{q_b}B_b v_{q_b}(\lambda\gamma_b)^{q_b}
        =q_bv_{q_b}2^{q_b}B_b\gamma_b^{q_b}\lambda^{q_b-1}.
\]
Because \(\lambda\le1\), this is at most the first term in \(A_{\rm grp}\).  When \(q_b=1\), the
tangential term is zero.  When \(q_b\ge2\), \Cref{cor:reftrace-inv-radius} gives
\[
        \E\frac1{\|P_b Y_t^{\rm ref}\|}
        \le C_{q_b}(2^{q_b}B_b)^{1/q_b}
        =2C_{q_b}B_b^{1/q_b}.
\]
Multiplication by \(\gamma_b(q_b-1)\) gives the second term in \(A_{\rm grp}\).  These bounds are
uniform in \(t\), so averaging over \([0,h]\) proves the reference-trace bound.  The estimate
\(M_\lambda= d/\lambda\) is the global trace bound.  The choices
\eqref{eq:group-lasso-lambda-choice} and \eqref{eq:group-lasso-step-choice} are \Cref{thm:main-complete} with
\(B_{\rm ref}\le A_{\rm grp}\) and \(M_\lambda= d/\lambda\), up to universal constants.  Since
\(d/(\varepsilon\lambda)=dG^2/(2\varepsilon^2)\), the complexity display follows.
\end{proof}

\begin{remark}[Block-size dependence in the group-lasso bound]
\label{rem:group-lasso-block-size}
The preceding proposition is meant to make the \(\varepsilon\)-dependence transparent: it gives a \(\lambda\)-free upper bound on \(B_{\rm ref}\).  The displayed constant should not be read as an optimized estimate in the block dimension.  Indeed, for a block of size \(q_b\), let
\[
    \alpha_b := 1-h(L_b+\lambda^{-1}),
    \qquad
    \overline B_b:=B_b\alpha_b^{-q_b},
    \qquad
    \kappa_b:=(v_{q_b}\overline B_b)^{1/q_b}.
\]
Before replacing \(\lambda\) and \(\alpha_b\) by crude constants, the contribution of this block satisfies
\[
    B_{{\rm ref},b}
    \le
    q_b\gamma_b\kappa_b(\lambda\gamma_b\kappa_b)^{q_b-1}
    +\mathbf 1_{\{q_b\ge2\}}q_b\gamma_b\kappa_b .
\]
The first term is the contribution of the central ball \(\{\|z_b\|\le\lambda\gamma_b\}\), while the second term controls the tangential curvature \(\gamma_b(q_b-1)/\|z_b\|\).  Thus, for $q_b\ge 2$, whenever \(\lambda\gamma_b\kappa_b\le1\), the central ball term is no larger than the tangential term, and it is in fact damped by the factor \((\lambda\gamma_b\kappa_b)^{q_b-1}\).  The simplified \(\lambda\)-free constant used in the proposition deliberately discards this damping in order to emphasize the \(\varepsilon^{-2}\) complexity.

The quantity \(\kappa_b\) is an inverse effective block radius, not an exponentially large density constant.  From the definition of \(B_b\), a ball lower bound on the normalizing integral gives
\[
    \kappa_b
    \le
    C\alpha_b^{-1}\left(\sqrt{\frac{L_b}{q_b}}+\frac{\gamma_b}{q_b}\right),
\]
with a universal constant \(C\).  Consequently the tangential term is at most
\[
    C\alpha_b^{-1}\left(\gamma_b\sqrt{L_bq_b}+\gamma_b^2\right).
\]
Hence, for bounded \(L_b\) and standard group-lasso weights, the active-trace constant grows polynomially, and often linearly or sublinearly, in the block size; it is not exponential in \(q_b\).

For the central ball term one may also use a one-dimensional slab bound.  For any unit vector \(u\) in the block,
\[
    \{\|Y_{b,t}^{\rm ref}\|\le\lambda\gamma_b\}
    \subseteq
    \{|\langle u,Y_{b,t}^{\rm ref}\rangle|\le\lambda\gamma_b\},
\]
so the slab estimate gives
\[
    \frac{q_b}{\lambda}
    \mathbb P\{\|Y_{b,t}^{\rm ref}\|\le\lambda\gamma_b\}
    \le
    \frac{2q_b\gamma_bB_u}{\alpha_u}.
\]
Therefore the central contribution can be bounded by the minimum of the ball estimate and this slab estimate.  The slab estimate is sometimes a convenient \(\lambda\)-free way to remove the apparent \(\gamma_b^{q_b}\) factor.  It does not, however, control the tangential term: the inverse-radius estimate still needs the \(q_b\)-dimensional small-ball bound when \(q_b\ge2\).
\end{remark}

\subsection{Generalized lasso and anisotropic total variation}\label{subsec:generalized-lasso-example}

Let \(D\in\R^{J\times d}\) have nonzero rows \(d_j^\top\), and consider
\begin{equation}\label{eq:generalized-lasso-penalty}
        g(x)=\gamma\|Dx\|_1,
        \qquad \gamma>0.
\end{equation}
A global Lipschitz constant is
\begin{equation}\label{eq:generalized-lasso-global-G}
        G=\gamma\|D^\top\|_{\infty\to2}
        =\gamma\sup_{\|z\|_\infty\le1}\|D^\top z\|.
\end{equation}
For each row define
\begin{equation}\label{eq:generalized-lasso-row-unit}
        u_j=\frac{d_j}{\|d_j\|},
\end{equation}
and
\begin{equation}\label{eq:generalized-lasso-row-constants}
        L_j=\sup_{x\in\R^d}u_j^\top\nabla^2f(x)u_j,
        \qquad
        G_j=\gamma\|Du_j\|_1,
\end{equation}
\begin{equation}\label{eq:generalized-lasso-density-constant}
        B_j=\left[\int_{\R}\exp\left\{-\frac{L_j}{2}s^2-2G_j|s|\right\}\,\dd s\right]^{-1}.
\end{equation}
The number \(G_j\) is the Lipschitz size of \(g\) in direction \(u_j\).  Finally, set
\begin{equation}\label{eq:generalized-lasso-Rj}
        R_j=\sum_{\ell=1}^J |d_j^\top d_\ell|.
\end{equation}

\begin{lemma}[Prox-cell trace bound for generalized lasso]
\label{lem:generalized-lasso-prox-cell-trace}
Let \(p_\lambda(x)=\prox_{\lambda g}(x)\).  For a point \(p\in\R^d\), define the active row set
\begin{equation}\label{eq:generalized-lasso-active-set}
        A(p)=\{j:\ d_j^\top p=0\}.
\end{equation}
For \(A\subseteq\{1,\ldots,J\}\), let \(D_A\) be the submatrix of \(D\) with rows indexed by
\(A\), and use the convention \(P_{\ker D_A}=\Id\) when \(A=\emptyset\).

A prox-cell means an open polyhedral region of the \(x\)-space on which the active set
\(A(p_\lambda(x))\) is fixed and, for every inactive row \(j\notin A(p_\lambda(x))\), the sign of
\(d_j^\top p_\lambda(x)\) is fixed.  At every differentiability point \(x\) of \(p_\lambda\) lying in
such a cell, with \(A=A(p_\lambda(x))\),
\begin{equation}\label{eq:generalized-lasso-prox-jacobian}
        Dp_\lambda(x)=P_{\ker D_A},
        \qquad
        a_\lambda(x)=\frac{\rank(D_A)}{\lambda}.
\end{equation}
Consequently, for a.e. \(x\),
\begin{equation}\label{eq:generalized-lasso-curvature-slab}
        a_\lambda(x)
        \le
        \frac1\lambda\sum_{j=1}^J
        \mathbf 1_{\{|d_j^\top x|\le \lambda\gamma R_j\}}.
\end{equation}
In particular, we can choose
\begin{equation}\label{eq:generalized-lasso-Mlambda}
        M_\lambda= \frac{\rank(D)}{\lambda}.
\end{equation}
\end{lemma}

\begin{proof}
We first explain the cell formula. Fix a prox-cell and write its active set as \(A\).  On this cell,
the signs
\[
        \sigma_j=\operatorname{sign}(d_j^\top p_\lambda(x)),\qquad j\notin A,
\]
are fixed, while the active rows satisfy \(D_Ap_\lambda(x)=0\). For \(x\) in the cell, write \(p=p_\lambda(x)\). Since
\(d_j^\top p\neq 0\) for every \(j\notin A\), there is a relative
neighborhood of \(p\) in \(\ker D_A\) on which
\[
        \operatorname{sign}(d_j^\top y)=\sigma_j,
        \qquad j\notin A.
\]
On this neighborhood, the original proximal objective restricted to
\(\ker D_A\) agrees with
\[
        y\longmapsto
        \frac{1}{2\lambda}\|x-y\|^2
        +\gamma\sum_{j\notin A}\sigma_j d_j^\top y.
\]
Hence \(p\) is a local minimizer of this function on \(\ker D_A\).
Since the function is strongly convex on \(\ker D_A\), this local
minimizer is its unique global minimizer. Completing the square
therefore gives
\begin{equation}\label{eq:generalized-lasso-cell-problem}
\begin{aligned}
        p_\lambda(x)
        &=
        \arg\min_{y:\,D_Ay=0}
        \left\{
        \frac{1}{2\lambda}\|x-y\|^2
        +\gamma\sum_{j\notin A}\sigma_j d_j^\top y
        \right\}                                                     \\
        &=
        P_{\ker D_A}
        \left(
        x-\lambda\gamma\sum_{j\notin A}\sigma_j d_j
        \right).
\end{aligned}
\end{equation}
Since \(A\) and the signs \(\sigma_j\) are fixed on the cell, the
right-hand side is affine in \(x\). Hence
\[
        Dp_\lambda(x)=P_{\ker D_A}
\]
throughout the cell.
By \Cref{lem:moreau-regularity},
\[
        H_\lambda(x)=\lambda^{-1}\bigl(\Id-Dp_\lambda(x)\bigr)
\]
at every differentiability point of \(p_\lambda\).  Since \(P_{\ker D_A}\) is the orthogonal
projection onto \(\ker D_A\), the matrix \(\Id-P_{\ker D_A}\) is the orthogonal projection onto
\((\ker D_A)^\perp=\operatorname{range}(D_A^\top)\), whose trace is \(\rank(D_A)\).  Therefore
\[
        a_\lambda(x)=\operatorname{tr}H_\lambda(x)
        =
        \frac{1}{\lambda}\operatorname{tr}(\Id-P_{\ker D_A})
        =
        \frac{\rank(D_A)}{\lambda}.
\]

It remains to locate the points where this curvature can occur.  The optimality condition for the
original proximal problem gives
\begin{equation}\label{eq:generalized-lasso-kkt}
        x-p_\lambda(x)=\lambda\gamma D^\top z,
        \qquad
        z_j\in \partial |(Dp_\lambda(x))_j|,
        \qquad
        |z_j|\le 1 .
\end{equation}
If \(j\in A(p_\lambda(x))\), then \(d_j^\top p_\lambda(x)=0\).  Multiplying
\eqref{eq:generalized-lasso-kkt} by \(d_j^\top\) gives
\[
        |d_j^\top x|
        =
        \lambda\gamma |d_j^\top D^\top z|
        \le
        \lambda\gamma\sum_{\ell=1}^J |d_j^\top d_\ell|\,|z_\ell|
        \le
        \lambda\gamma R_j .
\]
Thus every active row \(j\) forces \(x\) to lie in the row slab
\[
        \{|d_j^\top x|\le \lambda\gamma R_j\}.
\]
At differentiability points in the above cells,
\[
        a_\lambda(x)
        =
        \frac{\rank(D_{A(p_\lambda(x))})}{\lambda}
        \le
        \frac{|A(p_\lambda(x))|}{\lambda}
        \le
        \frac1\lambda
        \sum_{j=1}^J
        \mathbf 1_{\{|d_j^\top x|\le \lambda\gamma R_j\}} .
\]
Since \(g(x)=\gamma\|Dx\|_1\) is polyhedral convex, its proximal map
\(p_\lambda\) is piecewise affine; see, e.g.,
\cite[Proposition~12.30]{rockafellar1998variational}. It follows that,
outside a Lebesgue-null set, the active set and all inactive signs
are locally constant. Hence almost every \(x\) lies in a prox-cell,
and the preceding inequality holds for Lebesgue-a.e. \(x\). This
proves \eqref{eq:generalized-lasso-curvature-slab}. Finally,
\[
        a_\lambda(x)
        =
        \frac{\rank(D_{A(p_\lambda(x))})}{\lambda}
        \le
        \frac{\rank(D)}{\lambda}
\]
for a.e. \(x\).  This gives \eqref{eq:generalized-lasso-Mlambda}.
\end{proof}

\begin{proposition}[Generalized lasso reference trace]
\label{prop:generalized-lasso-reference-trace}
Let
\begin{equation}\label{eq:generalized-lasso-AD}
        A_D:=4\gamma\sum_{j=1}^J\frac{R_jB_j}{\|d_j\|}.
\end{equation}
There exist universal constants \(c,C>0\) such that the following holds.  For
\(0<\varepsilon\le1\), assume \(G>0\), set
\begin{equation}\label{eq:generalized-lasso-lambda-choice}
        \lambda=\frac{2\varepsilon}{G^2},
\end{equation}
and choose
\begin{equation}\label{eq:generalized-lasso-step-choice}
        h
        =c\min\left\{
              L_f^{-1},
              \frac{\varepsilon^2}{\tau_f+G^2+A_D},
              \frac{\varepsilon\lambda}{\rank(D)}
        \right\}.
\end{equation}
Then
\begin{equation}\label{eq:generalized-lasso-AD-bound}
        B_{\rm ref}\le A_D,
        \qquad
        M_\lambda= \frac{\rank(D)}{\lambda}.
\end{equation}
If
\begin{equation}\label{eq:generalized-lasso-N-choice}
        N\ge \frac{C}{mh}
        \log\left(1+\frac{m\Phi_0(\lambda,h)}{\varepsilon^2}\right),
\end{equation}
then \(\sqrt m W_2(\mu_N,\pi)\le\varepsilon\).  Equivalently, with the step size
\eqref{eq:generalized-lasso-step-choice}, it is sufficient to take
\begin{equation}\label{eq:generalized-lasso-complexity}
        N_D(\varepsilon)
        \le
        \frac{C}{m}
        \left[
              L_f+
              \frac{\tau_f+A_D+\rank(D)G^2}{\varepsilon^2}
        \right]
        \log\left(1+\frac{m\Phi_0(\lambda,h)}{\varepsilon^2}\right).
\end{equation}
\end{proposition}

\begin{proof}
By \Cref{lem:generalized-lasso-prox-cell-trace},
\[
        B_{\rm ref}
        \le
        \frac1{\lambda h}\sum_{j=1}^J\int_0^h
        \Pp\{|d_j^\top Y_t^{\rm ref}|\le \lambda\gamma R_j\}\,\dd t.
\]
The step size \eqref{eq:generalized-lasso-step-choice} implies \eqref{eq:examples-density-compatibility}.
By \Cref{cor:reftrace-slab}, 
\[
        \Pp\{|d_j^\top Y_t^{\rm ref}|\le \lambda\gamma R_j\}
        \le
        \frac{4\lambda\gamma R_jB_j}{\|d_j\|},
\]
uniformly in \(t\).  Substitution proves \(B_{\rm ref}\le A_D\).  The bound on
\(M_\lambda\) is \eqref{eq:generalized-lasso-Mlambda}.  The choices
\eqref{eq:generalized-lasso-lambda-choice} and \eqref{eq:generalized-lasso-step-choice} are
\Cref{thm:main-complete} with \(B_{\rm ref}\le A_D\) and
\(M_\lambda= \rank(D)/\lambda\), up to universal constants.  Since
\(\rank(D)/(\varepsilon\lambda)=\rank(D)G^2/(2\varepsilon^2)\), the complexity display follows.
\end{proof}

\begin{corollary}[One-dimensional anisotropic total variation]
\label{cor:anisotropic-tv-reference-trace}
Let \(D\in\R^{(d-1)\times d}\) be the first-difference matrix,
\begin{equation}\label{eq:first-difference-D}
        d_j=e_j-e_{j+1},
        \qquad j=1,\ldots,d-1.
\end{equation}
Then \(\|d_j\|=\sqrt2\), \(\rank(D)=d-1\), and \(R_j\le4\) for every row.  If
\(B_j\le B_{\rm tv}\) for every \(j\), set
\begin{equation}\label{eq:anisotropic-tv-Atv}
        A_{\rm tv}:=8\sqrt2\,\gamma B_{\rm tv}(d-1).
\end{equation}
There exist universal constants \(c,C>0\) such that, for \(0<\varepsilon\le1\) and \(G>0\), set
\begin{equation}\label{eq:anisotropic-tv-lambda-choice}
        \lambda=\frac{2\varepsilon}{G^2},
\end{equation}
and choose
\begin{equation}\label{eq:anisotropic-tv-step-choice}
        h
        =c\min\left\{
              L_f^{-1},
              \frac{\varepsilon^2}{\tau_f+G^2+A_{\rm tv}},
              \frac{\varepsilon\lambda}{d-1}
        \right\}.
\end{equation}
Then
\begin{equation}\label{eq:anisotropic-tv-bref}
        B_{\rm ref}\le A_{\rm tv},
        \qquad
        M_\lambda= \frac{d-1}{\lambda}.
\end{equation}
If
\begin{equation}\label{eq:anisotropic-tv-N-choice}
        N\ge \frac{C}{mh}
        \log\left(1+\frac{m\Phi_0(\lambda,h)}{\varepsilon^2}\right),
\end{equation}
then \(\sqrt m W_2(\mu_N,\pi)\le\varepsilon\).  In particular,
\begin{equation}\label{eq:anisotropic-tv-complexity}
        N_{\rm tv}(\varepsilon)
        \le
        \frac{C}{m}
        \left[
              L_f+
              \frac{\tau_f+A_{\rm tv}+dG^2}{\varepsilon^2}
        \right]
        \log\left(1+\frac{m\Phi_0(\lambda,h)}{\varepsilon^2}\right).
\end{equation}
\end{corollary}

\begin{proof}
For first-difference rows, \(\|e_j-e_{j+1}\|=\sqrt2\).  Also \(d_j^\top d_j=2\), and the only
nonzero cross-products with other rows are the neighboring ones, equal to \(-1\) when they exist.
Thus \(R_j\le2+1+1=4\).  Applying \eqref{eq:generalized-lasso-AD} and using
\(\sum_{j=1}^{d-1}B_j\le B_{\rm tv}(d-1)\) gives \(A_{\rm tv}\).  The step-size and iteration
statements are the specialization of \Cref{prop:generalized-lasso-reference-trace} with
\(\rank(D)=d-1\).
\end{proof}

\begin{remark}[Improved dependence on the target accuracy]
\label{rem:examples-global-trace-comparison}
The main gain in the examples above is that \(B_{\rm ref}\) is bounded
independently of \(\lambda\). If one instead used only the global bound
\(B_{\rm ref}\leq M_\lambda=O(\lambda^{-1})\), then the term
\(B_{\rm ref}/\epsilon^2\) in \Cref{thm:main-complete} would scale as
\(\epsilon^{-3}\) under the choice \(\lambda=2\epsilon/G^2\).
The structured active-trace estimates remove this cubic contribution.
The remaining term \(M_\lambda/\epsilon\) scales as \(\epsilon^{-2}\),
yielding the \(\widetilde O(\epsilon^{-2})\) dependence displayed above.
\end{remark}

\section{Conclusion}
\label{sec:conclusion}

We developed an active-trace analysis of the classical fixed-\(\lambda\) MYULA kernel for strongly log-concave composite targets. The main bound replaces global Moreau-curvature control by an occupation-weighted trace along a reference heat path. Combined with the Moreau-bias estimate, this gives end-to-end Wasserstein guarantees for the original nonsmooth target and, for the structured penalties considered here, \(\widetilde O(\varepsilon^{-2})\) accuracy dependence without changing the MYULA transition. The same occupation-weighted viewpoint may be useful for other smoothing-based Langevin methods.

\appendix

\section{Real-Analysis Primer for A.e. and Weak Hessians}
\label{app:ae-hessian}

This appendix fixes the real-analysis conventions underlying the
nonsmooth second-order notation used in the main text. We recall only the
standard facts needed to interpret the Hessian of a \(C^{1,1}\) function as
an a.e. derivative, equivalently as an \(L^\infty\) weak Hessian, and to use
the corresponding symmetry, positivity, and matrix bounds. Throughout the
paper, Hessians of Moreau envelopes are understood in this a.e./weak sense;
no everywhere classical \(C^2\) regularity is assumed. Standard references
for these facts include \cite{evans2025measure,leoni2017first}.

\subsection{Rademacher's theorem}
\label{app:rademacher}

A map $F:\R^{d_1}\to\R^{d_2}$ is locally Lipschitz continuous if for each compact $K\subset\R^d$ there is $L_K<\infty$ such that
\[
   \norm{F(x)-F(y)}\le L_K\norm{x-y},\qquad x,y\in K.
\]
Rademacher's theorem states that every locally Lipschitz continuous map between Euclidean spaces is
differentiable Lebesgue-a.e \cite[Theorem 3.2]{evans2025measure}.  In this paper it is used with
$F=\nabla V$, where $V\in C^{1,1}$.
Thus, although a Lipschitz gradient need not be differentiable everywhere, its
Jacobian exists outside a null set.

\subsection{A.e. Hessians, weak derivatives, and the $L^\infty$ weak Hessian}
\label{app:weak-derivatives}

Let $V\in C^{1,1}(\R^d)$.  By definition, $V\in C^1$ and $\nabla V$ is
Lipschitz.  At a point where the Lipschitz map $\nabla V$ is differentiable, we
define the classical a.e. Hessian by
\[
   \nabla^2V(x):=D(\nabla V)(x).
\]
Rademacher's theorem implies that this definition applies for a.e. $x$.

There is an equivalent weak-derivative interpretation.  For a locally integrable
function $u$, a function $w$ is the weak derivative $\partial_j u$ if
\[
   \int_{\R^d} u(x)\,\partial_j\varphi(x)\dd x
   =-\int_{\R^d} w(x)\varphi(x)\dd x
\]
for every smooth compactly supported test function $\varphi$.  When
$u=\partial_iV$, these weak derivatives form a matrix
$H_V=(H_{ij})$.  This matrix is called the weak Hessian of $V$.

Since $\nabla V$ is Lipschitz, its a.e.
Jacobian is bounded by $\Lip(\nabla V)$, and the a.e. Hessian belongs to
$L^\infty$. The weak Hessian and the a.e.
Jacobian $D(\nabla V)$ agree up to null sets, since the Lipschitz gradient is absolutely continuous on every line, allowing one-dimensional integration by parts.  Therefore, in the paper, the
phrases ``a.e. Hessian'' and ``$L^\infty$ weak Hessian'' refer to the same
matrix-valued object, viewed from two equivalent perspectives.

\subsection{Symmetry, convexity, and matrix bounds}
\label{app:hessian-matrix-bounds}

For a $C^2$ function, the Hessian is symmetric because mixed partial derivatives
commute.  For $V\in C^{1,1}$, the same statement holds in the weak sense:
\[
   H_{ij}=H_{ji}\qquad\text{a.e.}
\]
One way to see this is to mollify.  Let $\rho_\varepsilon$ be a smooth
mollifier and set $V_\varepsilon=\rho_\varepsilon*V$.  Then
$V_\varepsilon\in C^\infty$, so $\nabla^2V_\varepsilon$ is symmetric.  At every point where the Lipschitz map
\(\nabla V\) is differentiable, mollification recovers the derivative:
\[
\nabla^2V_\varepsilon(x)\to D(\nabla V)(x).
\]
Indeed, near such a point, \(\nabla V\) is well approximated by its linear
part, and the mollifier averages over a ball whose radius tends to zero.
Since \(\nabla V\) is differentiable a.e., the a.e. Hessian is the
pointwise a.e. limit of symmetric matrices. Therefore \(H_{ij}=H_{ji}\)
a.e.

If $V$ is convex, then the gradient is monotone:
\[
   \ip{\nabla V(x)-\nabla V(y)}{x-y}\ge0.
\]
At a point where $\nabla V$ is differentiable, take $y=x+tv$, divide by $t^2$,
and let $t\to0$.  This gives
\[
   v^\top\nabla^2V(x)v\ge0,
\]
so the a.e. Hessian is positive semidefinite.

If, in addition, $\Lip(\nabla V)\le L$, then at every differentiability point,
for every unit vector $v$,
\[
   \norm{D(\nabla V)(x)v}
   =\lim_{t\to0}\left\Vert\frac{\nabla V(x+tv)-\nabla V(x)}{t}\right\Vert
   \le L.
\]
Since the Hessian is symmetric and positive semidefinite, all its eigenvalues
lie in $[0,L]$.  Equivalently,
\[
   0\preceq \nabla^2V(x)\preceq L\Id_d\qquad\text{for a.e. }x.
\]

\section{Details for Moreau Weak Second-Order Regularity}
\label{app:moreau-second-order}

\begin{proof}[Justification of \Cref{lem:moreau-regularity}]

We first quote the standard first-order Moreau--Yosida facts. For a proper,
lower semicontinuous, convex function \(g\), the proximal map $p_\lambda$ is single-valued on \(\R^d\), and the Moreau envelope \(g_\lambda\) is
finite-valued, convex, and belongs to \(C^{1,1}(\R^d)\). Moreover,
\[
\nabla g_\lambda(x)=\frac{x-p_\lambda(x)}{\lambda},
\qquad
\Lip(\nabla g_\lambda)\le \lambda^{-1}.
\]
See, for example,
\cite[Chs.~12 and 23]{bauschke2020correction} and
\cite[Ch.~1.G]{rockafellar1998variational}.

It remains to interpret the second-order statements. Since \(g_\lambda\) is
convex and \(C^{1,1}\), the general facts reviewed in \Cref{app:ae-hessian} apply with
\(V=g_\lambda\) and \(L=\lambda^{-1}\). Hence there exists a set
\(E_\lambda\subset\R^d\), whose complement is Lebesgue-null, such that
\(\nabla g_\lambda\) is differentiable at every \(x\in E_\lambda\), and, with
\[
\nabla^2 g_\lambda(x):=D(\nabla g_\lambda)(x),
\]
the matrix \(\nabla^2 g_\lambda(x)\) is symmetric and satisfies
\[
0\preceq \nabla^2 g_\lambda(x)\preceq \lambda^{-1}I,
\qquad x\in E_\lambda.
\]

Finally, the gradient formula gives
\[
p_\lambda=\mathrm{Id}-\lambda\nabla g_\lambda.
\]
Therefore \(p_\lambda\) is differentiable at every \(x\in E_\lambda\), and
\[
Dp_\lambda(x)
=
I-\lambda\nabla^2 g_\lambda(x).
\]
Equivalently,
\[
\nabla^2 g_\lambda(x)
=
\lambda^{-1}\{I-Dp_\lambda(x)\}.
\]
\end{proof}

\begin{proof}[Proof of \Cref{lem:moreau-gradient-bound}]
Let $p=p_\lambda(x)$.  The proximal optimality condition gives
$\nabla g_\lambda(x)=(x-p)/\lambda\in\partial g(p)$.  It is therefore enough to show that every subgradient of \(g\) has norm at
most \(G\).
If $v\in\partial g(p)$ and
$u$ is a unit vector, the subgradient inequality and Lipschitzness imply, for
$t>0$,
\[
   g(p)+t\ip{v}{u}\le g(p+tu)\le g(p)+Gt .
\]
Hence $\ip{v}{u}\le G$. Taking
\(u=v/\|v\|\) when \(v\ne0\) yields \(\|v\|\le G\). Therefore
\(\|\nabla g_\lambda(x)\|\le G\). 
\end{proof}

\section{Proof of the Heat Identity for \texorpdfstring{$C^{1,1}$}{C1,1} Test Functions}
\label{app:heat-identity}

\begin{proof}[Proof of \Cref{lem:heat-energy}]
We use the Itô--Krylov formula, i.e. Itô's formula with generalized
second derivatives; see \cite[Ch.~2, Sec.~10, Thm.~1]{krylov1980controlled}.

Let \(L=\operatorname{Lip}(\nabla V)\). Since \(\nabla V\) is globally
\(L\)-Lipschitz, the weak Hessian satisfies
\[
   \|H_V(x)\|_{\op}\le L
   \qquad\text{for a.e. }x,
\]
and hence
\[
   |\Delta V(x)|\le dL
   \qquad\text{for a.e. }x.
\]
Moreover,
\[
   \|\nabla V(x)\|\le \|\nabla V(0)\|+L\|x\|,
\]
and the fundamental theorem of calculus along the segment from \(0\) to \(x\)
gives
\[
   |V(x)|\le C_V(1+\|x\|^2)
\]
for some constant \(C_V<\infty\). Thus \(V(Y_t)\) is integrable for
\(t\in[0,h]\), since \(Y_0\in L^2\).

For \(R>0\), let
\[
   \tau_R:=\inf\{t\ge 0:\|Y_t\|\ge R\}.
\]
On the ball \(B_R\), the function \(V\) satisfies the hypotheses of the
Itô--Krylov formula: its first derivatives are continuous and its generalized
second derivatives are locally square-integrable. Applying the formula to the
stopped process \(Y_{t\wedge\tau_R}\) gives
\[
  V(Y_{h\wedge\tau_R})-V(Y_0)
  =
  \sqrt{2}\int_0^{h\wedge\tau_R}
       \langle \nabla V(Y_t),dW_t\rangle
  +
  \int_0^{h\wedge\tau_R}\Delta V(Y_t)\,dt .
\]
The stochastic integral has mean zero, since the integrand is bounded on
\([0,\tau_R]\). Therefore
\[
  \E V(Y_{h\wedge\tau_R})-\E V(Y_0)
  =
  \E\int_0^{h\wedge\tau_R}\Delta V(Y_t)\,dt .
\]

Letting \(R\to\infty\), we have \(\tau_R\to\infty\) a.s. The left-hand side
converges to \(\E V(Y_h)-\E V(Y_0)\) by dominated convergence, using the
quadratic growth of \(V\) and
\[
   \E\sup_{0\le t\le h}\|Y_t\|^2<\infty .
\]
The right-hand side converges to
\[
   \int_0^h \E\,\Delta V(Y_t)\,dt
\]
by dominated convergence, since \(|\Delta V|\le dL\) a.e. This proves the
identity.

\end{proof}

\bibliographystyle{plain}
\bibliography{reference}

\begin{thebibliography}{10}

\bibitem{bauschke2020correction}
Heinz~H Bauschke and Patrick~L Combettes.
\newblock Correction to: convex analysis and monotone operator theory in hilbert spaces.
\newblock In {\em Convex analysis and monotone operator theory in Hilbert spaces}, pages C1--C4. Springer, 2020.

\bibitem{brosse2017sampling}
Nicolas Brosse, Alain Durmus, {\'E}ric Moulines, and Marcelo Pereyra.
\newblock Sampling from a log-concave distribution with compact support with proximal langevin monte carlo.
\newblock In {\em Conference on learning theory}, pages 319--342. PMLR, 2017.

\bibitem{dalalyan2026improved}
Arnak~S Dalalyan and Avetik Karagulyan.
\newblock Improved guarantees for langevin monte carlo with average smoothness.
\newblock {\em arXiv preprint arXiv:2605.31413}, 2026.

\bibitem{durmus2019analysis}
Alain Durmus, Szymon Majewski, and B{\l}a{\.z}ej Miasojedow.
\newblock Analysis of langevin monte carlo via convex optimization.
\newblock {\em Journal of Machine Learning Research}, 20(73):1--46, 2019.

\bibitem{durmus2017nonasymptotic}
Alain Durmus and Eric Moulines.
\newblock Nonasymptotic convergence analysis for the unadjusted langevin algorithm.
\newblock 2017.

\bibitem{durmus2018efficient}
Alain Durmus, Eric Moulines, and Marcelo Pereyra.
\newblock Efficient bayesian computation by proximal markov chain monte carlo: when langevin meets moreau.
\newblock {\em SIAM Journal on Imaging Sciences}, 11(1):473--506, 2018.

\bibitem{eftekhari2023forward}
Armin Eftekhari, Luis Vargas, and Konstantinos~C Zygalakis.
\newblock The forward--backward envelope for sampling with the overdamped langevin algorithm.
\newblock {\em Statistics and Computing}, 33(4):85, 2023.

\bibitem{ehrhardt2024proximal}
Matthias~J Ehrhardt, Lorenz Kuger, and Carola-Bibiane Sch{\"o}nlieb.
\newblock Proximal langevin sampling with inexact proximal mapping.
\newblock {\em SIAM Journal on Imaging Sciences}, 17(3):1729--1760, 2024.

\bibitem{evans2025measure}
Lawrence~C Evans.
\newblock {\em Measure theory and fine properties of functions}.
\newblock Chapman and Hall/CRC, 2025.

\bibitem{ghaderi2024smoothing}
Susan Ghaderi, Masoud Ahookhosh, Adam Arany, Alexander Skupin, Panagiotis Patrinos, and Yves Moreau.
\newblock Smoothing unadjusted langevin algorithms for nonsmooth composite potential functions.
\newblock {\em Applied Mathematics and Computation}, 464:128377, 2024.

\bibitem{habring2026diffusion}
Andreas Habring, Alexander Falk, Martin Zach, and Thomas Pock.
\newblock Diffusion at absolute zero: Langevin sampling using successive moreau envelopes.
\newblock {\em SIAM Journal on Imaging Sciences}, 19(1):35--77, 2026.

\bibitem{habring2024subgradient}
Andreas Habring, Martin Holler, and Thomas Pock.
\newblock Subgradient langevin methods for sampling from nonsmooth potentials.
\newblock {\em SIAM Journal on Mathematics of Data Science}, 6(4):897--925, 2024.

\bibitem{klatzer2024accelerated}
Teresa Klatzer, Paul Dobson, Yoann Altmann, Marcelo Pereyra, Jesus~Maria Sanz-Serna, and Konstantinos~C Zygalakis.
\newblock Accelerated bayesian imaging by relaxed proximal-point langevin sampling.
\newblock {\em SIAM Journal on Imaging Sciences}, 17(2):1078--1117, 2024.

\bibitem{krylov1980controlled}
Nicolai~V Krylov.
\newblock {\em Controlled diffusion processes}.
\newblock Springer, 1980.

\bibitem{lau2022bregman}
Tim Tsz-Kit Lau and Han Liu.
\newblock Bregman proximal langevin monte carlo via bregman-moreau envelopes.
\newblock In {\em International Conference on Machine Learning}, pages 12049--12077. PMLR, 2022.

\bibitem{leoni2017first}
Giovanni Leoni.
\newblock {\em A first course in Sobolev spaces}.
\newblock American Mathematical Soc., 2017.

\bibitem{lewis2002active}
Adrian~S Lewis.
\newblock Active sets, nonsmoothness, and sensitivity.
\newblock {\em SIAM Journal on Optimization}, 13(3):702--725, 2002.

\bibitem{liang2017activity}
Jingwei Liang, Jalal Fadili, and Gabriel Peyr{\'e}.
\newblock Activity identification and local linear convergence of forward--backward-type methods.
\newblock {\em SIAM Journal on Optimization}, 27(1):408--437, 2017.

\bibitem{liu2026proximal}
Linghai Liu and Sinho Chewi.
\newblock A proximal gradient algorithm for composite log-concave sampling.
\newblock {\em arXiv preprint arXiv:2605.12461}, 2026.

\bibitem{mou2022efficient}
Wenlong Mou, Nicolas Flammarion, Martin~J Wainwright, and Peter~L Bartlett.
\newblock An efficient sampling algorithm for non-smooth composite potentials.
\newblock {\em Journal of Machine Learning Research}, 23(233):1--50, 2022.

\bibitem{otto2000generalization}
Felix Otto and C{\'e}dric Villani.
\newblock Generalization of an inequality by talagrand and links with the logarithmic sobolev inequality.
\newblock {\em Journal of Functional Analysis}, 173(2):361--400, 2000.

\bibitem{pereyra2016proximal}
Marcelo Pereyra.
\newblock Proximal markov chain monte carlo algorithms.
\newblock {\em Statistics and Computing}, 26(4):745--760, 2016.

\bibitem{pereyra2020accelerating}
Marcelo Pereyra, Luis~Vargas Mieles, and Konstantinos~C Zygalakis.
\newblock Accelerating proximal markov chain monte carlo by using an explicit stabilized method.
\newblock {\em SIAM Journal on Imaging Sciences}, 13(2):905--935, 2020.

\bibitem{rockafellar1998variational}
R~Tyrrell Rockafellar and Roger~JB Wets.
\newblock {\em Variational analysis}.
\newblock Springer, 1998.

\bibitem{salim2019stochastic}
Adil Salim, Dmitry Kovalev, and Peter Richt{\'a}rik.
\newblock Stochastic proximal langevin algorithm: Potential splitting and nonasymptotic rates.
\newblock {\em Advances in Neural Information Processing Systems}, 32, 2019.

\bibitem{salim2020primal}
Adil Salim and Peter Richtarik.
\newblock Primal dual interpretation of the proximal stochastic gradient langevin algorithm.
\newblock {\em Advances in Neural Information Processing Systems}, 33:3786--3796, 2020.

\bibitem{tibshirani2012degrees}
Ryan~J Tibshirani and Jonathan Taylor.
\newblock Degrees of freedom in lasso problems.
\newblock 2012.

\bibitem{vaiter2017degrees}
Samuel Vaiter, Charles Deledalle, Jalal Fadili, Gabriel Peyr{\'e}, and Charles Dossal.
\newblock The degrees of freedom of partly smooth regularizers.
\newblock {\em Annals of the Institute of Statistical Mathematics}, 69(4):791--832, 2017.

\bibitem{villani2009optimal}
C{\'e}dric Villani et~al.
\newblock {\em Optimal transport: old and new}, volume 338.
\newblock Springer, 2009.

\end{thebibliography}

\end{document}